\documentclass[12pt,draftcls,onecolumn]{IEEEtran}

\ifCLASSINFOpdf
  \usepackage[pdftex]{graphicx}
\else
\fi
\usepackage[cmex10]{amsmath}
\usepackage[tight,footnotesize]{subfigure}

\usepackage{amsfonts,amsmath,amsthm,amssymb,graphicx,epstopdf,cite,subfigure}
\usepackage{amsfonts,amsmath,amsthm,amssymb,graphicx,epstopdf,cite,subfigure}
\usepackage{psfrag,amssymb,amsmath,pifont,cite,graphics,graphicx,epsfig,subfigure,amscd,helvet,multirow,enumerate}
\usepackage{color}
\usepackage{epstopdf}
\usepackage{accents}

\usepackage{amssymb}
\usepackage{mathtools}
\usepackage{tikz}
\usetikzlibrary{shapes.symbols,patterns}
\usetikzlibrary{shapes,shapes.multipart,shapes.geometric}
\usetikzlibrary{arrows.meta}

\usepackage{multirow}

\usepackage[linesnumbered,ruled,norelsize]{algorithm2e}
	\makeatletter
	\newcommand{\removelatexerror}{\let\@latex@error\@gobble}
	\makeatother

\usepackage{booktabs}
\usepackage{pgfplots}
\usepackage{epstopdf}
\usepackage{scalefnt}
\usepackage{adjustbox}
\usepackage{ifthen}
\usetikzlibrary{shapes,shapes.geometric}
\usetikzlibrary{calc,backgrounds}

\usepackage{amssymb}
\usepackage{mathtools}
\usepackage{tikz}
\usepackage{tikz-3dplot}
\usetikzlibrary{shapes.symbols,patterns}
\usetikzlibrary{shapes,shapes.multipart,shapes.geometric}
\usetikzlibrary{arrows.meta}
\usetikzlibrary{shapes,shapes.geometric}
\usetikzlibrary{calc,backgrounds}

\usepackage{multirow}

\usepackage{hhline}

\newcommand{\mathbff}[3]{\mathbf{#1}_{#2}^{(#3)}}

\newcommand\Tstrut{\rule{0pt}{4.4ex}}       
\newcommand\Bstrut{\rule[-1.35ex]{0pt}{0pt}} 
\newcommand{\TBstrut}{\Tstrut\Bstrut} 

\usepackage{amsthm}

\usepackage{url} 
\usepackage{colortbl}
\usepackage{hhline}
\newtheorem{theorem}{Theorem}[section]
\newtheorem{lemma}[theorem]{Lemma}

\newtheorem{corollary}[theorem]{Corollary}

\newtheorem{remark}[theorem]{Remark}
\newenvironment{prooff}{\textit{Proof of Theorem \ref{fl:thm002}:}}{\hfill$\square$}

\usepackage{color,soul}
\soulregister\cite7
\soulregister\ref7
\soulregister\eqref7

\begin{document}

\newcommand{\matc}[2][ccccccccccccccccccc]{\left[
\begin{array}{#1}
#2\\
\end{array}
\right]}
\newcommand{\matr}[2][rrrrrrrrrrrrrrrrrrrrrrrr]{\left[
\begin{array}{#1}
#2\\
\end{array}
\right]}
\newcommand{\matl}[2][lllllllllllllllllll]{\left[
\begin{array}{#1}
#2\\
\end{array}
\right]}

\newcommand{\pt}[2]{P_{T_{\mathbf{Y}_{#1}}\widetilde{\mathcal{Y}}}(#2)}

\newcommand*\circled[1]{\tikz[baseline=(char.base)]{
  \node[shape=circle,draw,inner sep=2pt] (char) {#1};}}


{\Large \textbf{Notice:} This work has been submitted to the IEEE for possible publication. Copyright may be transferred without notice, after which this version may no longer be accessible.}

\clearpage


%
\title{Gradual Federated Learning with Simulated Annealing}

\author{\IEEEauthorblockN{Luong Trung Nguyen, Junhan Kim, and Byonghyo Shim}

\IEEEauthorblockA{Information System Laboratory\\
Department of Electrical and Computer Engineering, Seoul National University\\
Email: \{ltnguyen,junhankim,bshim\}@islab.snu.ac.kr}
\thanks{A part of this paper was presented at the International Conference on Acoustics, Speech, and Signal Processing (ICASSP), 2021~\cite{localization:luongICASSP}.

This work was supported by Samsung Research Funding \& Incubation Center for Future Technology of Samsung Electronics under Project Number SRFC-IT1901-17.}
}


%


\maketitle


\begin{abstract}
    Federated averaging (FedAvg) is a popular federated learning (FL) technique that updates the global model by averaging local models and then transmits the updated global model to devices for their local model update. One main limitation of FedAvg is that the average-based global model is not necessarily better than local models in the early stage of the training process so that FedAvg might diverge in realistic scenarios, especially when the data is non-identically distributed across devices and the number of data samples varies significantly from device to device. In this paper, we propose a new FL technique based on simulated annealing. The key idea of the proposed technique, henceforth referred to as \textit{simulated annealing-based FL} (SAFL), is to allow a device to choose its local model when the global model is immature. Specifically, by exploiting the simulated annealing strategy, we make each device choose its local model with high probability in early iterations when the global model is immature. From extensive numerical experiments using various benchmark datasets, we demonstrate that SAFL outperforms the conventional FedAvg technique in terms of the convergence speed and the classification accuracy.
\end{abstract}



%
\IEEEpeerreviewmaketitle

\section{Introduction}


Federated learning (FL) is an emerging distributed learning technique where hundreds or thousands of devices jointly train a common machine learning (ML) model without exchanging their local dataset with the centralized server or other devices~\cite{mcmahan2017,sahu2020,semiari2019,nguyen2019tut,samarakoon2019}.
A wide range of FL applications include human face recognition, next-word prediction, resource allocation, device tracking, basestation association, cyberattack detection, to name just a few~\cite{yang2019,wang2019,sattler2019,wanghan2019}.
In the FL-based approach, a learning task is performed in an iterative fashion, mainly following by three steps (see Fig. \ref{fl:fig001}). First, a server sets up a common ML model and then broadcasts the model to the user devices. Second, user devices train the model locally and individually using their own local datasets. Third, the server evaluates the model by aggregating the locally trained parameters sent by the devices.

The central challenge of FL is to improve the learning capability of user devices without sharing their own datasets with other devices. In fact, due to various reasons such as user privacy and limited resources (e.g., computing hardware, battery power, network capacity, bandwidth), data generated in one device cannot be transmitted to the server or other devices. One well-known approach to deal with this issue is federated averaging (FedAvg)~\cite{mcmahan2017}. In this technique, instead of transmitting data, each device transmits locally trained parameters (e.g., gradients or updated model parameters) to the server. The server updates the global model by averaging the local parameters and then sends the updated model back to devices for the local model update.

\begin{figure}[t]
\centering
\begin{adjustbox}{max width=0.7\textwidth, max totalheight=\textheight,keepaspectratio}
\begin{tikzpicture}
\node (c0) {\includegraphics[scale=1]{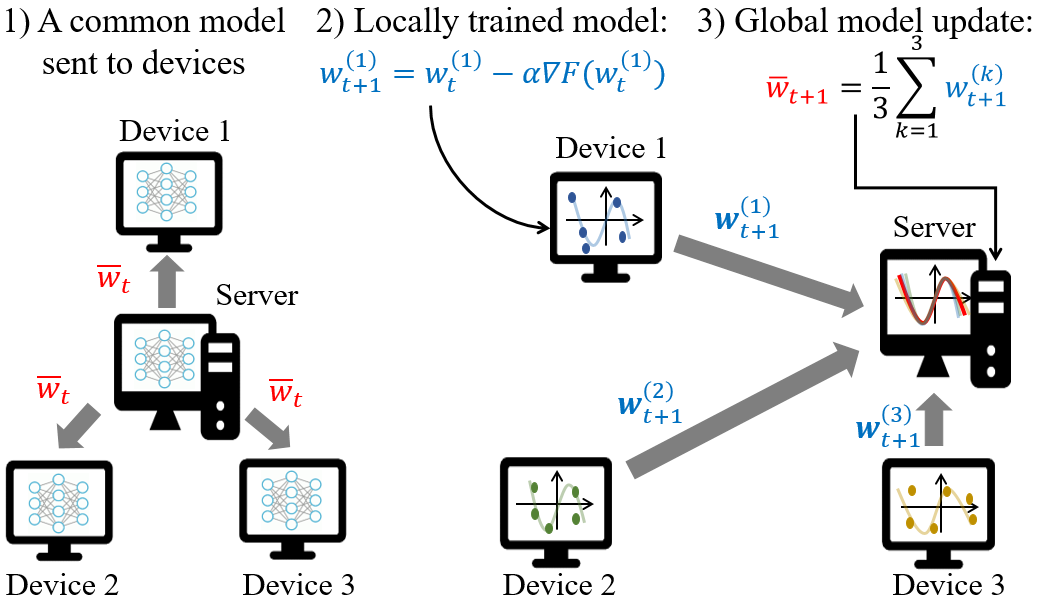}};
\node[below of=c0, xshift=-11cm, yshift=-8.5cm] (c1) {\Huge (a)};
\node[right of=c1,xshift=17cm] (c2) {\Huge (b)};
\end{tikzpicture}
\end{adjustbox}
\caption{Federated learning operation: (a) centralized server sends a common model to user devices and (b) each device locally trains the model using its own local data and then upload the trained network parameters to the server to globally update the model.}
\label{fl:fig001}
\end{figure}

While FedAvg is effective in solving nonconvex problem, it has been shown that FedAvg and its variants might diverge in realistic scenarios where the data is non-identically distributed across devices (e.g., data of different languages in the next-word prediction application) and/or the number of data samples significantly varies from device to device~\cite{sahu2020, li2020federated}.
One important reason for the divergence of FedAvg is that the average-based global model is not necessarily better than locally trained models so that just relying on the global model might degrade the entire learning process~\cite{yuchen2012,yossi2015,martin2010}. To illustrate this, we consider a simple FL task whose goal is to minimize the cost function given by
\begin{equation}
J(\mathbf{w}; \mathcal{D}) = \sum\limits_{(\mathbf{x}_i,y_i)\in\mathcal{D}}(y_i-\mathbf{x}_i^T\mathbf{w})^2 + \|\mathbf{w}\|_1,
\label{fl:eq035}
\end{equation}
where $\mathcal{D} = \{(\mathbf{x}_i, y_i)\}_i$ is the training dataset and $\|\mathbf{w}\|_1 = |w_1| + |w_2|$ is the $\ell_1$-norm of $\mathbf{w}$. For simplicity, we consider two devices with the local datasets $\mathcal{D}_1 = \{( [\frac{1}{4} \ 0]^T, -1)\}$ and $\mathcal{D}_2 =\{([0 \ \frac{3}{2}]^T, 1)\}$. One can easily check that the parameters $\mathbf{w}$ minimizing the cost function $J(\mathbf{w}; \mathcal{D})$ with respect to $\mathcal{D}_1$ and $\mathcal{D}_2$ are $\mathbf{w}^{(1)} = [0 \ 0]^T$ and $\mathbf{w}^{(2)} = [0 \ \frac{4}{9}]^T$, respectively (see Appendix \ref{fl:apxF}). Using $\mathbf{w}^{(1)}$ and $\mathbf{w}^{(2)}$, we obtain the average-based model $\overline{\mathbf{w}}$ evaluated at the server: $\overline{\mathbf{w}} = \frac{1}{2}(\mathbf{w}^{(1)} + \mathbf{w}^{(2)}) = [0 \ \frac{2}{9}]^T$. Since the optimum weight over $\mathcal{D}_{1} \cup \mathcal{D}_{2}$ is $\mathbf{w}_\ast =  [0 \ \frac{4}{9}]^T$, we have
\begin{align*}
    \|\mathbf{w}_\ast - \overline{\mathbf{w}}\|_2
    = \frac{2}{9}
    > 0
    = \|\mathbf{w}_\ast - \mathbf{w}^{(2)}\|_2,
\end{align*}
which implies that the average value $\overline{\mathbf{w}}$ is worse than the locally generated value $\mathbf{w}^{(2)}$. In this scenario, clearly, it would be better for the second device to use its own solution $\mathbf{w}^{(2)}$ instead of the server feedback $\overline{\mathbf{w}}$. Simply put, the moral of the story is that collaboration might do more harm than good, especially when things are not ready.

\begin{figure}[t]
\centering
\begin{adjustbox}{max width=0.6\textwidth,max totalheight=\textheight,keepaspectratio}
\begin{tikzpicture}

\coordinate (x0) at (0,5);
\coordinate (x1) at (1,4.5);
\coordinate (x2) at (2,5.5);

\coordinate (x3) at (3,3);
\coordinate (x4) at (5.25,4);

\coordinate (x5) at (6.5,1);
\coordinate (x6) at (8,4.5);
\coordinate (o) at (0,0);
\coordinate (x) at (9,0);
\coordinate (y) at (0,6.5);

\coordinate (b1) at (4.3,3.7);
\coordinate (b2) at (3.6,3.2);
\coordinate (b3) at (5.1,4.05);

\coordinate (e1) at (4,6.5);
\coordinate (e2) at (4,6);
\coordinate (e3) at (4,5.5);

\draw [<->, very thick] (y) -- node [rotate=90, above, yshift=0.15cm] {Aggregated cost function $F(\mathbf{w})$} (o) -> node [below, yshift=-0.15cm] {Model parameter $\mathbf{w}$} (x);

\filldraw[color = black!30] (b1) circle  (0.15);
\filldraw [color = black!70] (b2) circle (0.15);
\filldraw [color = black!50] (b3) circle (0.15);

\filldraw[color = black!30] (e1) circle  (0.15) node[right, xshift=0.2cm] {\small \textcolor{black}{Current model of user device}};
\filldraw [color = black!70] (e2) circle (0.15) node[right, xshift=0.2cm] {\small \textcolor{black}{Sever feedback (accepted with prob. $1-p$)}};
\filldraw [color = black!50] (e3) circle (0.15) node[right, xshift=0.2cm] {\small \textcolor{black}{Perturbation update (accepted with prob. $p$)}};

\begin{scope}[on background layer]
\draw [very thick, color=blue!60!black] plot [smooth] coordinates {(x0) (x1) (x2) (x3) (x4) (x5) (x6)};
\end{scope}

\node[below of=x5,yshift=0.7cm] (d1) {\small Global minimum};
\node[below of=x3,yshift=0cm,xshift=-1cm] (d2) {\small Local minimum trap};

\draw[->,thick] (d2.north) to [out=90,in=-90] ([yshift=-0.1cm]x3);
\draw[->,thick] ([yshift=0.1cm]b1.north) to [out=180,in=90] node[above,pos=.9,yshift=0.2cm] {$1-p$}  ([yshift=0.2cm]b2.north);
\draw[->,thick] ([yshift=0.1cm]b1.north) to [out=90,in=180] node[above,pos=.6] {$p$}  ([yshift=0.1cm,xshift=-0.15cm]b3.west);

\end{tikzpicture}
\end{adjustbox}
\caption{The SA-based update strategy of the proposed SAFL. Here, the local model selection probability $p$ is set to exponentially decay with respect to the training iteration. }
\label{fl:fig100}
\end{figure}
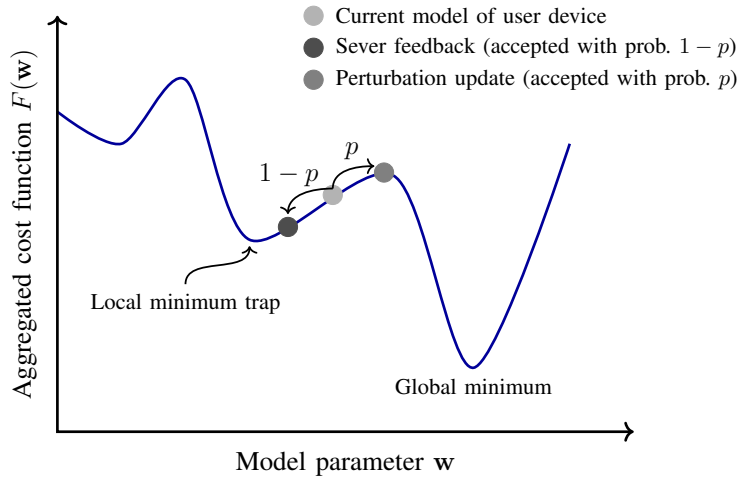


Our intent in this paper is to put forth a simple yet effective FL strategy overcoming the problem we mentioned.
Key idea of the proposed approach, referred to as the simulated annealing-based FL (SAFL), is that we encourage each device to stay with its locally trained model instead of relying on the collaborative learning model in the early stage of the learning process.   
When the collaborative model becomes mature and reliable after the reasonable number of iterations, we use the server-generated model to update the device.
This idea can be well explained using the simulated annealing (SA) strategy.
In the SA strategy, the solution space is searched by imposing perturbations on the estimates of parameters~\cite{kirkpatrick1983,eglese1990, locatelli2000, chenzhu2011, handong2019}.
In the early stage (a.k.a., \textit{heating stage}), the SA algorithm decides to move the system to a new (presumably perturbation) state with high probability, even though the new state might not be better than the current state, to avoid the chance of trapping in the local optima.
In the later stage (a.k.a., \textit{cooling stage}), the SA algorithm reduces the exploration of the perturbation space. 

Inspired by the SA strategy, the proposed SAFL updates the local model of each user probabilistically. To be specific, SAFL decides whether the device keeps its own locally updated model with some modification (i.e., perturbation update) or uses the global evaluation model provided by the server (i.e., server feedback) (see Fig. \ref{fl:fig100}).
In the early iterations where the global model is immature, we give a favor to the locally updated model by setting the local model selection probability high.
As the number of iterations increases, we gradually reduce this probability so that the device relies more on the server feedback, which helps to avoid the overfitting to the local dataset. 

The main contributions of this paper are summarized as follows:
\begin{itemize}
    \item We propose a new FL technique called SAFL (Section~\ref{fl:sec001}) inspired by the SA technique. From extensive numerical experiments on various datasets including MNIST, Fashion-MNIST, CIFAR-10, and Google speech commands, we demonstrate that the proposed SAFL technique is effective and in fact outperforms the conventional FedAvg technique by a large margin in terms of accuracy and convergence speed (Section~\ref{fl:sec002}). Specifically, in the MNIST dataset, SAFL converges two times faster than FedAvg and also achieves more than $50\%$ improvement in the classification accuracy.
   
    \item We analyze the performance of the proposed SAFL technique (Section~\ref{sec:convergence analysis}). Specifically, we show that under some suitable conditions, the mean squares error (MSE) of SAFL satisfies
    \begin{align*}
        E[\| \widehat{\mathbf{w}} - \mathbf{w}_{*} \|_{2}^{2}]
        &\le \xi
    \end{align*}
    after $\mathcal{O}(\frac{1}{\xi})$ iterations where $\widehat{\mathbf{w}}$ is the evaluated model parameters and $\mathbf{w}_{*}$ is the optimal model parameters (see Theorem~\ref{fl:thm002}).
   
   

\item We extend SAFL to the scenario where the performance of the average-based global model is degraded due to non-i.i.d. data and data imbalance among devices (Section~\ref{sec:Extension of SAFL_Outliers}). Our key idea is to detect biased local updates by measuring the performance gap between the global and local models. Specifically, if the performance gap is large, then we consider the local update as a biased update and do not upload it to the server. In doing so, we can exclude the biased local update in the update of the global model and prevent the performance degradation of the global model. From the numerical results, we demonstrate that the extended SAFL is effective in handling the non-i.i.d. data and reducing the number of local updates uploaded to the server (see Section~\ref{fl:sec002}). 
\end{itemize}

We briefly summarize notations used in this paper.
For a vector $\mathbf{a} \in \mathbb{R}^{n}$, $\text{Diag}(\mathbf{a}) \in \mathbb{R}^{n \times n}$ is the diagonal matrix formed by $\mathbf{a}$.
$\| \mathbf{a} \|_2$ stands for the spectral norm (i.e., the largest singular value) of $\mathbf{a}$.
The inner product of two vectors $\mathbf{a}$ and $\mathbf{b}$ is defined as $<\mathbf{a},\mathbf{b}> = \mathbf{a}^T\mathbf{b}$.
$\mathbf{A} \odot \mathbf{B}$ is the Hadamard product (or element-wise multiplication) of two matrices $\mathbf{A}$ and $\mathbf{B}$.
Given a function $f:\mathbf{X}\in\mathbb{R}^{n_1\times n_2}\rightarrow f(\mathbf{X})\in\mathbb{R}$, $\nabla_{\mathbf{X}}f(\mathbf{X})$ is the Euclidean gradient of $f(\mathbf{X})$ with respect to $\mathbf{X}$, i.e., $\left[\nabla_{\mathbf{X}}f(\mathbf{X})\right]_{ij}=\frac{\partial f(\mathbf{Y})}{\partial y_{ij}}$.
$\mathbf{1} = \matc{1 & 1 & \cdots & 1}^T$ is all-ones vector.

\section{Proposed SAFL Algorithm}
\label{fl:sec001}


We consider a communication system consisting of one central server and $n$ user devices. The server generates a global model with parameters $\mathbf{w}$ and then transmits the generated model to $s$ selected devices $(1\leq s\leq n)$. 
Each selected device has its own dataset $\mathcal{D}^{(k)} = \{ (\mathbf{x}_{i}^{(k)}, y_{i}^{(k)}) \}_{i=1}^{m_{k}}$ to train the local model, where $\mathbf{x}_{i}^{(k)} \in \mathbb{R}^{q}$ is an input data sample (e.g., image), $y_{i}^{(k)}$ is the class label of $\mathbf{x}_{i}^{(k)}$, and $m_{k}$ is the number of data samples in the $k$-th device. We consider the standard FL setting where devices cannot exchange their own datasets with other devices or the central server. In each iteration, FedAvg updates the model parameters $\mathbf{w}$ (e.g., weights and biases) by taking the following steps. First, using its own dataset $\mathcal{D}^{(k)}$, each user device updates the model parameters locally to minimize the loss function $F(\mathbf{w}; \mathcal{D}^{(k)})$.\footnote{For example, if the mean squared error (MSE) is employed as a loss function, then $F(\mathbf{w}; \mathcal{D}) = \frac{1}{|\mathcal{D}|} \sum_{(\mathbf{x}, y) \in \mathcal{D}} \frac{1}{2}(\mathbf{x}^{T} \mathbf{w} - y)^{2}$.} For example, the update expression of the model parameters $\mathbf{z}_{t}^{(k)}$ at the $k$-th device is
\begin{align} \label{eq:gradient descent_locally updated parameters}
    \mathbf{z}_{t}^{(k)}
    &= \mathbf{w}_{t-1}^{(k)} - \alpha \left . \frac{\partial F(\mathbf{w}; \mathcal{D}^{(k)})}{\partial \mathbf{w}} \right |_{\mathbf{w}=\mathbf{w}_{t-1}^{(k)}},
\end{align}
where $\alpha$ is the learning rate and $\mathbf{w}_{t-1}^{(k)}$ is the local model parameters after $t-1$ iterations. Second, the server aggregates the local updates $\mathbf{z}_{t}^{(k)}$ to evaluate the global model parameters. The update expression of the global evaluation model is 
\begin{equation}
\bar{\mathbf{z}}_{t} = \sum_{k=1}^{n}\eta_k \mathbf{z}_{t}^{(k)},
\label{eq:eq001}
\end{equation}
where $\eta_k$ is the coefficient satisfying $\sum_k\eta_k = 1$.\footnote{We consider the generic setting of $\eta_k$ which is an arbitrary value defined by user. A typical setting of $\eta_k$ is $\eta_k = \frac{m_k}{\sum_k m_k}$~\cite{mcmahan2017}.} Note that when $s < n$, we simply set $\eta_k = 0$ for non-selected devices. Finally, the server transmits the globally updated parameters $\bar{\mathbf{z}}_{t}$ to the selected devices to update the local models. That is, the local model parameters $\mathbf{w}_{t}^{(k)}$ is updated as 
\begin{equation}
\mathbf{w}_{t}^{(k)} = \left\lbrace \begin{matrix}
\overline{\mathbf{z}}_{t} & \text{if the device receives } \overline{\mathbf{z}}_{t}\\
\mathbf{z}_{t}^{(k)} & \text{otherwise}
\end{matrix} \right. .
\label{fl:eq300}
\end{equation}


One potential drawback of the conventional FedAvg technique is that an entire FL process can be degraded by applying the hard-decision rule in \eqref{fl:eq300}. This is because the global evaluation model $\bar{\mathbf{z}}_{t}$ is not necessarily better than locally updated parameters $\mathbf{z}_{t}^{(k)}$ in many practical scenarios. 
For example, in the next word prediction application, a language model is trained to predict which word comes next when the initial text fragment is given.
In heterogeneous scenarios, users with different countries might use their own mother languages with different grammar and word combination rules (e.g., a subject-verb-object (SVO) rule is used in English, while a subject-object-verb (SOV) rule is used in Korean). Since the next word prediction task is performed with different language rules, the average-based model might perform much worse than the locally trained language model of a local device.

As another example, one can consider the face and object recognition problem where a classification model is trained to identify the user’s face ID.
The local dataset collected from user's personal images is often non-i.i.d. distributed across devices. 
Since the global model is aggregated by averaging the locally trained models, it may overfit to the local data.
In this case, if the device uses the average-based model exclusively, the device might also suffer the overfiting problem, even when the good training dataset is available.
Indeed, it has been shown that FedAvg can diverge in such non-i.i.d. scenario~\cite{sahu2020, li2020federated}.

\begin{figure*}[t]
\centering
\begin{adjustbox}{max width=\textwidth,max totalheight=\textheight,keepaspectratio}
\begin{tikzpicture}

\node (c0) {\includegraphics[scale=0.55]{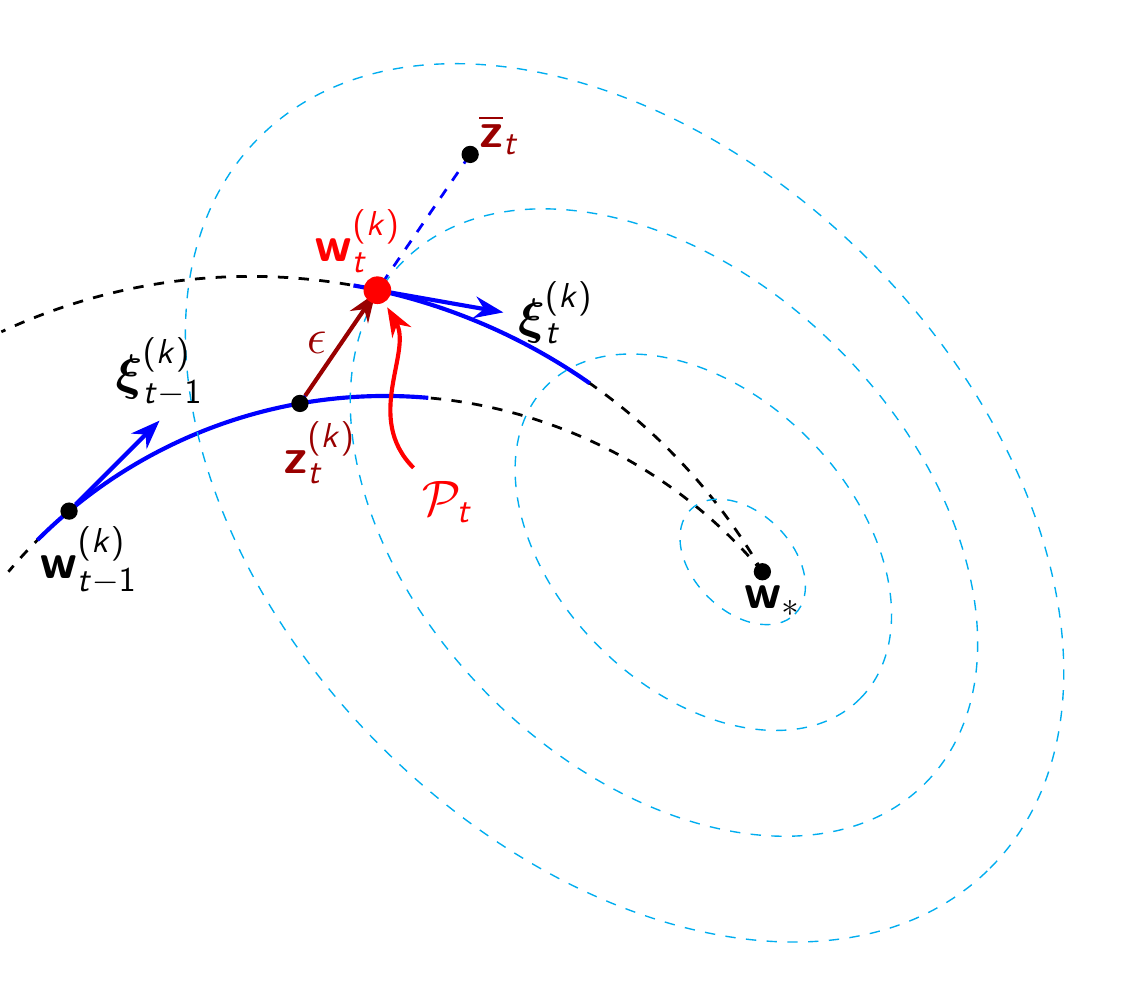}};
\node[below of=c0,yshift=-2cm] {\large (a)};
\node[right of=c0,xshift=5.5cm] (c1) {\includegraphics[scale=0.55]{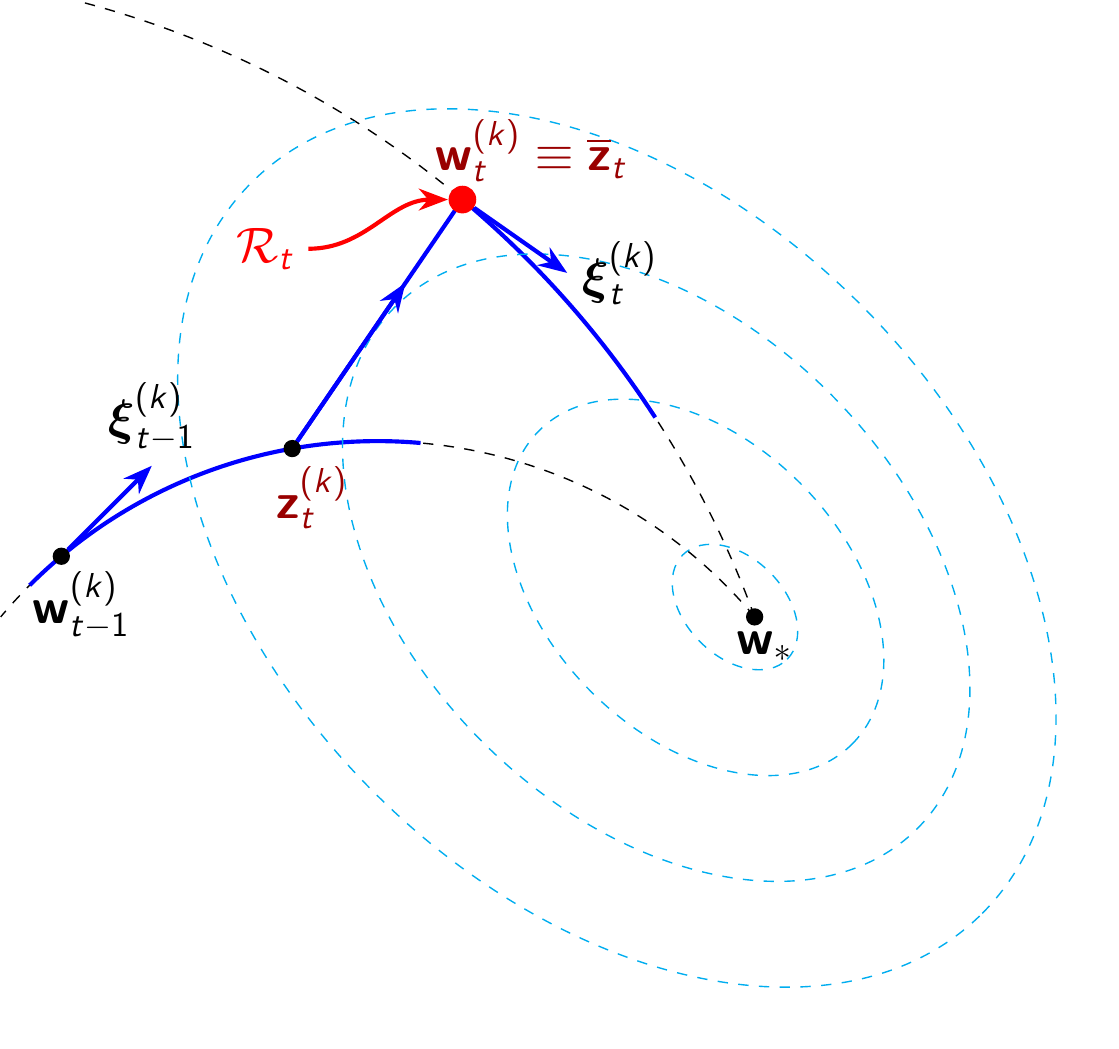}};
\node[below of=c1,yshift=-2cm] {\large (b)};

\end{tikzpicture}
\end{adjustbox}
\caption{In each iteration of the proposed SAFL technique, each device (a) moves to the perturbation state $\mathcal{P}_{t}^{(k)}$ with probability $p$ or (b) stays in the normal state $\mathcal{R}_{t}^{(k)}$ with probability $1-p$. In this figure, $\boldsymbol\xi_{t}^{(k)}$ is a descent direction.}
\label{fl:fig002}
\end{figure*}

\begin{table}[t!]
    \centering
    \caption{The proposed SAFL Algorithm} \label{tab:SAFL}
{
\begin{tabular}{l}
\hline
\textbf{Algorithm 1:} Proposed SAFL \\
\hline
\textbf{Input:} $T$: max iteration\\
\qquad\quad $E$: max local epoch\\
\qquad\quad $L$: control parameter\\
\qquad\quad $\{\eta_k\}_k$: weight coefficients\\
\qquad\quad $\{\mathbff{w}{0}{k}\}_k$: parameter initialization of the devices\\
\qquad\quad $s$: number of selected devices each round\\
\qquad\quad $t = 1$: initial iteration\\
\hline
\textbf{While} $t < T$ and a stopping criterion is not met \textbf{do:}\\
\qquad \textbf{For} the server \textbf{do:}\\
\qquad\qquad \textbf{If} the server receives $\mathbf{z}_t^{(k)}$ from the devices \textbf{then do:} \\
\qquad\qquad\qquad $\overline{\mathbf{z}}_t = \sum\limits_{k=1}^n \eta_k \mathbf{z}_t^{(k)}$ \\
\qquad\qquad\qquad Select a random set of devices $S_t$ satisfying $|S_t| = s$  \\
\qquad\qquad\qquad Send $\overline{\mathbf{z}}_t$ to $S_t$ \\
\qquad\qquad \textbf{End If}\\
\qquad \textbf{End For}\\

\qquad \textbf{For} device $k\in S_t$ \textbf{in parallel do:}\\
\qquad\qquad \textbf{For} e = 1 to E \textbf{do:}\\
\qquad\qquad\qquad \textbf{For all} example $A_{t}^{(i_k)}$, $i_k\in\{1,2,\cdots,m_k\}$ \textbf{do:}\\
\qquad\qquad\qquad\qquad $\mathbff{z}{t}{k} =\mathbff{w}{t-1}{k} - \alpha\nabla F_k(\mathbff{w}{t-1}{k};A_{t}^{(i_k)})$ \\
\qquad\qquad\qquad\qquad \textbf{If} the device receives $\overline{\mathbf{z}}_t$ from the server \textbf{then do:} \\
\qquad\qquad\qquad\qquad\qquad Generate $\mathbff{u}{t}{k}$ using \eqref{eq:definition of u}\\
\qquad\qquad\qquad\qquad\qquad $\mathbff{w}{t}{k} =  \mathbff{u}{t}{k}\odot\overline{\mathbf{z}}_{t} + (\mathbf{1}- \mathbff{u}{t}{k})\odot\mathbff{z}{t}{k}$ \\
\qquad\qquad\qquad\qquad \textbf{Else do:}  \\
\qquad\qquad\qquad\qquad\qquad $\mathbff{w}{t}{k} = \mathbff{z}{t}{k}$ \\
\qquad\qquad\qquad\qquad \textbf{End If}  \\
\qquad\qquad\qquad\qquad $t = t + 1$\\
\qquad\qquad\qquad \textbf{End For}\\
\qquad\qquad \textbf{End For}\\
\qquad\qquad Send $\mathbf{z}_t^{(k)}$ to the server\\
\qquad \textbf{End For}\\
\textbf{End While}\\
\hline
\textbf{Output:} $\widehat{\mathbf{w}}_t = \sum_{k}\mathbf{w}^{(k)}_t$\\
\hline
\end{tabular}
}
\end{table}

Inspired by this observation, we first define a weighted sum model that incorporates $\mathbf{z}_{t}^{(k)}$ and $\overline{\mathbf{z}}_{t}$. The corresponding local update model is expressed as
\begin{equation}
\mathbf{w}_{t}^{(k)} = \left\lbrace \begin{matrix}
\epsilon \overline{\mathbf{z}}_{t} + (1-\epsilon)\mathbff{z}{t}{k} & \text{if the device receives } \overline{\mathbf{z}}_{t}\\
\mathbf{z}_{t}^{(k)} & \text{otherwise}
\end{matrix} \right. .
\label{fl:eq107}
\end{equation}
where $\epsilon$ is the regularization parameter used to control the contribution of the global evaluation model in \eqref{fl:eq107}.
For example, by setting $\epsilon = 1$, the update expression \eqref{fl:eq107} is returned to the conventional FL case.
Whereas, by setting $\epsilon = 0$, the device ignores the server feedback $\overline{\mathbf{z}}_{t}$ and continues to use the locally trained model $\mathbf{z}_{t}^{(k)}$.

For the model selection, we consider a strategy inspired by the SA algorithm.
In the proposed SAFL, we define the normal state $\mathcal{R}_t$ and the perturbation state $\mathcal{P}_t$ as $\overline{\mathbf{z}}_{t}$ and $\epsilon \overline{\mathbf{z}}_{t} + (1-\epsilon)\mathbff{z}{t}{k}$, respectively (see Fig. \ref{fl:fig002}).
%
$\mathcal{P}_t$ is accepted with probability $p = \exp(-\frac{t}{L})$ where $L$ is a positive constant (a.k.a., the maximum temperature of SA~\cite{eglese1990}), while $\mathcal{R}_t$ is with probability $1-p$. 
To be specific, let $\mathbf{u}_{t}^{(k)}$ be the random vector whose $j$-th element $u_{j}$ satisfies
\begin{align} \label{eq:definition of u}
    u_{j}
    & = \begin{cases}
    \epsilon & \text{with probability } p=\exp \left ( -\frac{t}{L} \right ),\\
    1 & \text{with probability } 1-p,
    \end{cases}
\end{align}
then the local update expression \eqref{fl:eq107} can be reformulated as (see Fig.~\ref{fl:fig002})
\begin{align} \label{eq:local update formula}    
    \mathbf{w}_{t}^{(k)} & = \left\lbrace \begin{matrix}
\mathbf{u}_{t}^{(k)} \odot \bar{\mathbf{z}}_{t} + (\mathbf{1} - \mathbf{u}_{t}^{(k)}) \odot \mathbf{z}_{t}^{(k)} & \text{if receives } \overline{\mathbf{z}}_{t}\\
\mathbf{z}_{t}^{(k)} & \text{otherwise}
\end{matrix} \right. .
\end{align}
Note that the model selection probability $p$ decays exponentially with the number of iteration. In early iterations (i.e., $p$ is close to one), each device relies on its locally trained model and thus the local model would be trained mainly by the local dataset. In later iterations (i.e., $p$ is close to zero), the device uses the global evaluation model which is presumably more robust to the overfitting problem than the locally trained model.

We note that the server update procedure of SAFL is essentially the same as the conventional FedAvg so that various fusion models can be easily integrated to SAFL~\cite{chensun2019, yeganehfarshad2020, jisaravirta2021, jipanlong2019, jiangjilong2020, wuliangwang2020, huangchuzhou2021}.
For example, if we integrate the inverse distance aggregation (IDA) fusion model~\cite{yeganehfarshad2020} and SAFL, the coefficient $\eta_k$ is expressed as~\cite{jisaravirta2021} 
\begin{align}
\eta_k = \frac{\|\overline{\mathbf{z}}_t-\mathbf{z}_t^{(k)}\|_2^{-1}}{\sum\limits_{k=1}^n \|\overline{\mathbf{z}}_t-\mathbf{z}_t^{(k)}\|_2^{-1}}.  
\end{align} 

In Algorithm~\ref{tab:SAFL}, we summarize the proposed SAFL algorithm.

\section{Convergence Analysis of SAFL}
\label{sec:convergence analysis}


In this section, we analyze the convergence behavior of the proposed SAFL. 
%
%
For simplicity, we consider the scenario where each participating device updates its local model using the stochastic gradient descent (SGD)~\cite{bottou2018}.
Let $\delta_t$ be a user-predefined value satisfying 
\begin{align} \label{eq:definition of d}
    \delta_t
    & = \begin{cases}
    1 & \text{if the $k$-th device receives the server feedback }\overline{\mathbf{z}}_t ,\\
    0 & \text{else} ,
    \end{cases}
\end{align}
Then, the update expressions \eqref{eq:gradient descent_locally updated parameters} and \eqref{eq:local update formula} can be reformulated as
\begin{eqnarray}
\label{eq:local training formula}
\mathbff{z}{t}{k} & = & \mathbff{w}{t-1}{k} - \alpha\nabla F_k(\mathbff{w}{t-1}{k};A_{t}^{(i_k)}) ,\\
\label{eq:local update analysis}
\mathbf{w}_{t}^{(k)}
    & = & \delta_t\mathbf{u}_{t}^{(k)} \odot \bar{\mathbf{z}}_{t} + (\mathbf{1} - \delta_t\mathbf{u}_{t}^{(k)}) \odot \mathbf{z}_{t}^{(k)}.
\end{eqnarray}
where the input data $A_{t}^{(i_k)} = (\mathbf{x}_{i_k},y_{i_k})\in\mathcal{D}^{(k)}$ is sampled identically and independently at each iteration. Here, we put no assumption on the data distribution so that our analysis results can be applied for both i.i.d. and non-i.i.d. scenarios. Also note that $F_k(\mathbff{w}{t-1}{k};A_{t}^{(i_k)})$ is the cost function with respect to the data sample $A_{t}^{(i_k)}$ and $F_k(\mathbff{w}{t-1}{k})$ is the empirical risk function defined as 
\begin{equation}
F_k(\mathbff{w}{t-1}{k}) = \frac{1}{|\mathcal{D}^{(k)}|} \sum\limits_{i_k} F_k(\mathbff{w}{t-1}{k};A_{t}^{(i_k)}).
\end{equation}   

Before proceeding, we summarize the assumptions used in our analysis:

\begin{itemize}
\item[\textbf{A1}] $F_k(\mathbf{w})$ is non-negative: $F_k(\mathbf{w})\geq 0$ and $F_k(\mathbf{w}_\ast) = 0$ 
\item[\textbf{A2}] $F_k(\mathbf{w})$ is a smooth convex function: $\lambda\mathbf{I}\succeq \nabla^2 F_k(\mathbf{w}) \succeq \mu \mathbf{I}$ for $\lambda \geq \mu \geq 0$.
\item[\textbf{A3}] The stochastic gradient $\nabla F_k(\mathbff{w}{t}{k};A_{t}^{(i_k)} )$ has a bounded variance:
\begin{align}
tr(Var(\nabla F_k(\mathbff{w}{t-1}{k}) - \nabla F_k(\mathbff{w}{t-1}{k};A_{t}^{(i_k)})|\mathbff{w}{t-1}{k})) & \leq  \sigma_k^2 .
\end{align}


\end{itemize}

It is worth mentioning that these assumptions are used in various machine learning problems, such as linear regression, Tikhonov regularization, logistic regression, and support vector machine (SVM)~\cite{wanghan2019,yuchen2012,yossi2015}.

Without loss of generality, we focus on the minimization problem\footnote{The maximization problem can be converted into a minimization problem with the same solution by multiplying the objective function by $-1$.} of the empirical risk. Hence, \textbf{A1} ensures that the objective function is to be minimized to zero. When the objective function has a nonzero lower bound, say, $F_k(\mathbf{w}) \geq F_0$ for some constant $F_0$, we simply define a new objective function $\widetilde{F}_k(\mathbf{w}) = F_k(\mathbf{w}) - F_0$ and easily extend the analysis results to $\widetilde{F}_k(\mathbf{w})$.
Assumption \textbf{A2} is popularly used to guarantee a linear convergence rate of many gradient descent-based machine learning techniques~\cite{bottou2018}. Equivalently, \textbf{A2} can be expressed as~\cite{bottou2018}
\begin{itemize}
\item[\textbf{A2a}] $\nabla F_k(\mathbf{w})$ is $\lambda$-Lipschitz continuous:
\begin{align}
\|\nabla F_k(\mathbf{w}_2) - \nabla F_k(\mathbf{w}_1)\|_2  \leq \lambda\|\mathbf{w}_2-\mathbf{w}_1\|_2, \forall \mathbf{w}_1, \mathbf{w}_2.
\end{align}
\item[\textbf{A2b}] $F_k(\mathbf{w})$ is a $\mu$-strongly convex function:
\begin{align}
F_k(\mathbf{w}_2) & \geq F_k(\mathbf{w}_1) + \nabla F_k(\mathbf{w}_1)^T(\mathbf{w}_2  - \mathbf{w}_1)\nonumber\\
& \qquad + \frac{\mu }{2}\|\mathbf{w}_2 - \mathbf{w}_1\|_2^2, \forall \mathbf{w}_1, \mathbf{w}_2
\end{align}
for $\lambda\geq \mu \geq 0$.
\end{itemize}
Intuitively, \textbf{A2} ensures that there exists a quadratic lower bounds on the growth of the objective function. In our analysis, we use \textbf{A2}, together with Taylor's expansion, to build a universal upper bound on the MSE of the local updates $\mathbf{w}_t^{(k)}$.
Assumption \textbf{A3} is referred to as bounded variance condition in the literature~\cite{bottou2018}, which is widely used in the SGD convergence analysis~\cite{li2020federated,jiang2018,yu2019}.

In our main theorem, under \textbf{A1}, \textbf{A2}, and \textbf{A3}, we show that the proposed SAFL converges linearly\footnote{A sequence $\{u_t\}_{t=1}^\infty$ is said to converge linearly to $u_\ast$ if there exists a number $\lambda\in(0,1)$ such that $\lim\limits_{t\rightarrow\infty}\frac{|u_{t+1}-u_\ast|}{|u_t-u_\ast|}=\lambda$. Also, if $\lambda = 1$, then the sequence is said to converge sublinearly to $u_\ast$.} to an accurate solution.

\begin{theorem}
Under \textbf{A1}, \textbf{A2}, and \textbf{A3}, the MSE error bound of SAFL satisfies
\begin{align}
E[\|\widehat{\mathbf{w}}_t - \mathbf{w}_\ast\|_2^2] & \leq  (1-\alpha\mu)^{2t}\zeta  
 +  \frac{\alpha}{\mu} \sum\limits_{k=1}^n  \eta^2_k\sigma_k^2  \frac{1 - (1-\alpha\mu)^{2q}}{1 -  e^{-c}(1-\alpha\mu)^{2q}} ,
\label{fl:eq209}
\end{align}
where $\widehat{\mathbf{w}}_t = \sum\limits_{k=1}^n\eta_k \mathbf{w}_t^{(k)}$, $\zeta = \max\limits_k E[\|\mathbff{w}{0}{k} - \mathbf{w}_\ast\|_2^2]$, $q$ is the largest number of local iterations and $c = (1 - p(1 - \epsilon^2))(\frac{1   - \alpha  (2\lambda - \mu)}{1-\alpha\mu})^2$ for some $p$ and $\epsilon$, provided that $\alpha < \frac{1}{2\lambda-\mu}$.
\label{fl:thm002}
\end{theorem}

\begin{remark}
The right-hand side of \eqref{fl:eq209} consists of two terms: 1) the first term $(1-\alpha\mu)^{2t}\zeta$ converges linearly to zero with the iteration $t$ and 2) the second term is a function of the learning rate $\alpha$ and can be reduced with a small $\alpha$. In fact, when $\alpha = \mathcal{O}(\frac{1}{t})$, we can further show that SAFL converges sublinearly to the optimal solution.
%
%
%
\begin{corollary}
Under the same conditions of Theorem \ref{fl:thm002}, if $\alpha_t = \frac{\alpha_0}{t+1}$ for some $\alpha_0$  satisfying $ \frac{2-\sqrt{2}}{\mu}< \alpha_0 < \frac{2+\sqrt{2}}{\mu}$, then the MSE bound of SAFL satisfies
\begin{equation}
E[\|\widehat{\mathbf{w}}_t - \mathbf{w}_\ast\|_2^2] \leq \frac{c}{t+1},
\label{fl:eq206}
\end{equation}
where $c = \max\{\frac{2\alpha_0^2(\max_k  \sigma_k^2)}{2 - (2 - \mu\alpha_0)^2},\max\limits_k E[\|\mathbff{w}{0}{k} - \mathbf{w}_\ast\|_2^2]\}$. 
\label{fl:col001}
\end{corollary}
\begin{proof}
See Appendix \ref{fl:apxAA}.
\end{proof}
One can see that the MSE of the proposed SAFL scales in the order of $\mathcal{O}(\frac{1}{t})$.
This MSE bound matches to the latest results of federated optimization bound~\cite{li2020federated,jiang2018,yu2019}.
\end{remark} 

\begin{remark}
In Theorem \ref{fl:thm002}, the impact of the network size $n$ on the MSE bound is captured by the factor $\sum\limits_{k=1}^n \eta_k^2 \sigma_k^2$. In particular, when the local dataset has the same size (i.e., $\eta_k = \frac{1}{n}$), we have
\begin{align}
E[\|\widehat{\mathbf{w}}_t - \mathbf{w}_\ast\|_2^2] & \leq  (1-\alpha\mu)^{2t}\zeta   + \frac{\alpha\overline{\sigma}^2}{n\mu}  \frac{1 - (1-\alpha\mu)^{2q}}{1 -  e^{-c}(1-\alpha\mu)^{2q}} \nonumber\\
& \overset{(a)}{\leq}  (1-\alpha\mu)^{2t}\zeta   + \frac{\alpha\overline{\sigma}^2}{n\mu}  \nonumber ,
\label{fl:eq208}
\end{align} 
where $\overline{\sigma}^2 = \frac{1}{n}\sum\limits_{k=1}^n \sigma_k^2$ and (a) is because $\frac{1 - (1-\alpha\mu)^{2q}}{1 -  e^{-c}(1-\alpha\mu)^{2q}}\leq 1$.
When $t$ is large and $\alpha$ is fixed, this MSE bound decays and converges to $\frac{\alpha\overline{\sigma}^2}{n\mu}$, which means that the quality of the SAFL solution improves with the number of participating devices $n$.

\end{remark}

\begin{remark}
\label{fl:rmk001}
While we use the convexity assumption \textbf{A2} to facilitate our analysis, our main result can be readily extended to the case where $F_k(\mathbf{w})$ is not necessarily a strong convex function. 
For example, we consider the non-negative function $F_k(\mathbf{w})$ satisfying
\begin{itemize}
\item[\textbf{A4}] $F_k(\mathbf{w})$ is $\mu$-strongly quasi-convex:
\begin{eqnarray}
<\nabla F_k(\mathbf{w}),\mathbf{w} - \mathbf{w}_\ast> & \geq &  \frac{\mu}{2}\|\mathbf{w} - \mathbf{w}_\ast\|^2_2 .
\end{eqnarray}
\item[\textbf{A5}] The stochastic gradient $\nabla F_k(\mathbff{w}{t}{k};A_{t}^{(i_k)} )$ has a bounded variance:
\begin{eqnarray}
E[\|\nabla F_k(\mathbff{w}{t-1}{k};A_{t}^{(i_k)})\|_2^2] & \leq & \sigma_k^2 .
\end{eqnarray}
\end{itemize}
Note that \textbf{A4} is weaker than \textbf{A2} since if \textbf{A2} holds true, then we have
\begin{eqnarray}
 <\nabla F_k(\mathbf{w}),\mathbf{w} - \mathbf{w}_\ast> 
& \overset{(a)}{=} & F_k(\mathbf{w}) - F_k(\mathbf{w}_\ast) + \frac{1}{2}(\mathbf{w}-\mathbf{w}_\ast)^T\nabla^2F_k(\boldsymbol\xi)(\mathbf{w}-\mathbf{w}_\ast)\nonumber \\ 
& \overset{(b)}{\geq} & \frac{1}{2}(\mathbf{w}-\mathbf{w}_\ast)^T\nabla^2F_k(\boldsymbol\xi)(\mathbf{w}-\mathbf{w}_\ast)\nonumber \\
& \overset{(c)}{\geq} &  \frac{\mu}{2}\|\mathbf{w}_\ast - \mathbf{w}\|^2_2, \nonumber
\end{eqnarray}  
where $\boldsymbol\xi$ is a point between $\mathbf{w}$ and $\mathbf{w}_\ast$, (a) is due to Taylor's expansion, (b) is because $F_k(\mathbf{w}) - F_k(\mathbf{w}_\ast)\geq 0$, and (c) is because $\nabla^2 F_k(\boldsymbol\xi) \succeq \mu \mathbf{I}$.
We also note that \textbf{A4} does not imply \textbf{A2}, meaning that the quasi-strong convexity does not imply the convexity of $F_k(\mathbf{w})$~\cite{gower2019}. For a complete review of the functional classes satisfying this condition, see~\cite{necoara2019}.
Interestingly, using \textbf{A4} instead of \textbf{A2}, one can show that the proposed SAFL still has the same convergence rate $\mathcal{O}(\frac{1}{t})$.
\begin{theorem}
Under \textbf{A1}, \textbf{A4}, and \textbf{A5}, if $\alpha_t = \frac{\alpha_0}{t+1}$ for some $\alpha_0$  satisfying $ \alpha_0 > \frac{1}{\mu}$, the MSE bound of SAFL satisfies
\begin{equation}
E[\sum\limits_k\eta_k \|\mathbf{w}^{(k)}_t - \mathbf{w}_\ast\|_2^2] \leq \frac{c}{t+1},
\label{fl:eq214}
\end{equation}
where $c = \max\{\frac{\alpha_0^2  \sum_k\eta_k\sigma_k^2 }{\mu\alpha_0-1},E[\sum\limits_k\eta_k \|\mathbf{w}^{(k)}_0 - \mathbf{w}_\ast\|_2^2]\}$. 
\label{fl:thm003}
\end{theorem}
\begin{proof}
See Appendix \ref{fl:apxAB}.
\end{proof}
\end{remark}

\begin{remark}
In Theorem \ref{fl:thm002}, we put no constraint on the number of local iterations $q$ (a.k.a., the synchronization interval~\cite{yu2019}). Therefore, the communication rounds required for $T$ iterations is $\mathcal{O}(Tq^{-1})$, which is comparable to the latest results of existing distributed SGD techniques~\cite{yu2019}. 
\end{remark}

We are now ready to prove Theorem \ref{fl:thm002}.

\begin{prooff}
In our proof, we first show that the bound of $E[\|\widehat{\mathbf{w}}_t - \mathbf{w}_\ast\|_2^2]$ is expressed in terms of $E[\| \mathbff{w}{t}{k} - \mathbf{w}_\ast\|_2^2]$ and $\|E[ \mathbff{w}{t}{k} - \mathbf{w}_\ast]\|_2^2$. We then build the upper bounds for each of these. 
That is, 
\begin{align}
 E[\|\widehat{\mathbf{w}}_t - \mathbf{w}_\ast\|_2^2]  
& =  E[\|\sum_k\eta_k\mathbff{w}{t}{k} - \mathbf{w}_\ast\|_2^2]\notag\\
& =  \sum\limits_k \eta_k^2 E[\| \mathbff{w}{t}{k} - \mathbf{w}_\ast\|_2^2] \notag\\
& + \sum\limits_i\sum\limits_{j\neq i} \eta_i\eta_j E[<\mathbff{w}{t}{i} - \mathbf{w}_\ast,\mathbff{w}{t}{j} - \mathbf{w}_\ast>]\notag\\
& \overset{(a)}{\leq}  \sum\limits_k \eta_k^2 E[\| \mathbff{w}{t}{k} - \mathbf{w}_\ast\|_2^2] \notag \\
& + \sum\limits_i\sum\limits_{j\neq i} \eta_i\eta_j \|E[\mathbff{w}{t}{i} - \mathbf{w}_\ast]\|_2\|E[\mathbff{w}{t}{j} - \mathbf{w}_\ast]\|_2 \notag\\
& \overset{(b)}{=}  (\sum\limits_k \eta_k^2) \max\limits_k E[\| \mathbff{w}{t}{k} - \mathbf{w}_\ast\|_2^2] \notag\\
& + (1 - \sum\limits_k \eta_k^2)\max\limits_k \|E[ \mathbff{w}{t}{k} - \mathbf{w}_\ast]\|_2^2 ,
\label{fl:eq205}
\end{align}
where (a) is from the Cauchy-Schwarz inequality and (b) is because $\sum_i\sum_{j\neq i} \eta_i\eta_j = (\sum_k\eta_k)^2 - \sum_k\eta_k^2 = 1 - \sum_k\eta_k^2$.

In the following lemmas, we provide the upper bounds of $E[\| \mathbff{w}{t}{k} - \mathbf{w}_\ast\|_2^2]$ and $\|E[ \mathbff{w}{t}{k} - \mathbf{w}_\ast]\|_2^2$.

\begin{lemma}
Under the same conditions of Theorem \ref{fl:thm002}, we have
\begin{align}
E[\|\mathbff{w}{t}{k} - \mathbf{w}_\ast \|_2^2]   & \leq  
(1-\alpha\mu)^{2t}e^{-c\lfloor\frac{t}{q} \rfloor}  \zeta_0  
 +  \frac{\alpha^2\sigma_k^2}{1 - (1-\alpha\mu)^{2}} \frac{1 - (1-\alpha\mu)^{2q}}{1 -  e^{-c}(1-\alpha\mu)^{2q}}  ,
\label{fl:eq037}
\end{align}
where $\zeta_0 = \max\limits_i E[\|\mathbff{w}{0}{k} - \mathbf{w}_\ast\|_2^2]$.
\label{fl:lm002}
\end{lemma}
\begin{proof}
See Appendix \ref{fl:apxA}.
\end{proof}

\begin{lemma}
Under the same conditions of Theorem \ref{fl:thm002}, we have
\begin{equation}
\|E[\mathbf{w}^{(k)}_t - \mathbf{w}_\ast]\|_2 \leq (1-\alpha\mu)^{t} \sqrt{\zeta}  ,
\label{fl:eq204} 
\end{equation} 
where $\zeta = \max\limits_k E[\|\mathbff{w}{0}{k} - \mathbf{w}_\ast\|_2^2]$.
\label{fl:lm003} 
\end{lemma}
\begin{proof}
See Appendix \ref{fl:apxD}.
\end{proof}
Finally, using \eqref{fl:eq205}, \eqref{fl:eq037}, and \eqref{fl:eq204}, we have
\begin{align}
 E[\|\widehat{\mathbf{w}}_t - \mathbf{w}_\ast\|_2^2] 
& \leq   \sum\limits_k \eta_k^2 (1-\alpha\mu)^{2t}e^{ - c\lfloor\frac{t}{q} \rfloor}  \zeta + (1-\sum\limits_k \eta_k^2)(1-\alpha\mu)^{2t}\zeta  \nonumber\\
&   + \sum\limits_k  \eta^2_k\sigma_k^2   \frac{\alpha^2}{1 - (1-\alpha\mu)^{2}} \frac{1 - (1-\alpha\mu)^{2q}}{1 -  e^{-c}(1-\alpha\mu)^{2q}}\nonumber \\
& =   (1 - \sum\limits_k \eta_k^2 (1 - e^{ - c\lfloor\frac{t}{q} \rfloor} )  )(1-\alpha\mu)^{2t}\zeta  \nonumber\\
&   + \sum\limits_k  \eta^2_k\sigma_k^2   \frac{\alpha^2}{1 - (1-\alpha\mu)^{2}} \frac{1 - (1-\alpha\mu)^{2q}}{1 -  e^{-c}(1-\alpha\mu)^{2q}}\nonumber \\
& \overset{(a)}{\leq}   (1-\alpha\mu)^{2t}\zeta  +    \frac{\sum\limits_k  \eta^2_k\sigma_k^2\alpha^2}{1 - (1-\alpha\mu)^{2}} \frac{1 - (1-\alpha\mu)^{2q}}{1 -  e^{-c}(1-\alpha\mu)^{2q}}\nonumber \\
& \overset{(b)}{\leq}   (1-\alpha\mu)^{2t}\zeta  +    \frac{\alpha}{\mu} \sum\limits_k  \eta^2_k\sigma_k^2  \frac{1 - (1-\alpha\mu)^{2q}}{1 -  e^{-c}(1-\alpha\mu)^{2q}}\nonumber ,
\end{align}
where (a) is because $1 - \sum\limits_k \eta_k^2 (1 - e^{ - c\lfloor\frac{t}{q} \rfloor} )  \leq 1$ and (b) is because $1 - (1-\alpha\mu)^{2} = \alpha\mu(2-\alpha\mu) \geq \alpha\mu$, which establishes Theorem \ref{fl:thm002}.
\end{prooff}

\section{Extended SAFL for The Overfitting Problem}
\label{sec:Extension of SAFL_Outliers}

\begin{figure}[t]
\centering
\begin{adjustbox}{max size={!}{!}}
\begin{tikzpicture}
\node (c1) {\includegraphics[scale=0.45]{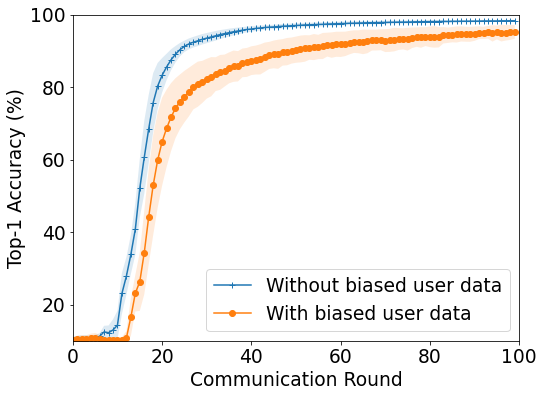}};
%
%
\end{tikzpicture}
\end{adjustbox}
\caption{The performance of the FL network with and without biased user data.}
\label{fl:fig004}
\end{figure}

In many practical scenarios, the FL performance can be degraded due to various reasons such as biased user data, training failures, model poisoning attack, and adversarial attacks~\cite{yang2019,wang2019,sattler2019,wanghan2019,li2020federated,bagdasaryan2018}.
For example, when a user device trains its local model using non-representative data (i.e., certain elements in the dataset are more heavily weighted and represented than others), then the local update of the device might cause a model overfitting problem, resulting in the degradation of the entire FL network performance.
To illustrate this behavior, we consider a FL network performing the MNIST classification (see Section~\ref{fl:sec002} for the detailed setting of the FL network).
Depending on the number of digit labels in the local datasets, devices can be classified into two groups: 1) a group with local dataset containing multiple digits (say, 1, 2, 5, 7, and 9) and 2) a group with dataset containing only one digit (say, 2).
Due to the data bias, devices in the second group can only learn features of one digit and, as a result, locally trained models might fail to predict other digits (1, 5, 7, and 9). In fact, when the locally trained models of the second group are overfitted to the biased dataset, there would be a performance degradation in the global evaluation model. In our example, if the server uses the local updates of the second group in the update of the global model, then the accuracy of the global model is degraded significantly (see Fig.~\ref{fl:fig004}).

\begin{table}[t!]
\label{fl:tab004}
\centering
{
\begin{tabular}{l}
\hline
\textbf{Algorithm 2:} Extended SAFL \\
\hline
\textbf{Input:} $T$: max iteration\\
\qquad\quad $E$: max local epoch\\
\qquad\quad $L$: maximum temperature\\
\qquad\quad $\{\eta_k\}_k$: weight coefficients\\
\qquad\quad $\{\mathbff{w}{0}{k}\}_k$: parameter initialization of the devices\\
\qquad\quad $s$: number of selected devices each round\\
\qquad\quad $t = 1$: initial iteration\\
\qquad\quad $\{q_0^{(k)}\}_k = 1$: initial probability of the local update\\
\hline
\textbf{While} $t < T$ and a stopping criterion is not met \textbf{do:}\\
\qquad \textbf{For} the server \textbf{do:}\\
\qquad\qquad \textbf{If} the server receives $\mathbf{z}_t^{(k)}$ from the devices \textbf{then do:} \\
\qquad\qquad\qquad $\overline{\mathbf{z}}_t = \sum\limits_{k=1}^{n} \eta_k \mathbf{z}_t^{(k)}$ \\
\qquad\qquad\qquad Select a random set of devices $S_t$ satisfying $|S_t| = s$  \\
\qquad\qquad\qquad Send $\overline{\mathbf{z}}_t$ to $S_t$ \\
\qquad\qquad \textbf{End If}\\
\qquad \textbf{End For}\\
\qquad \textbf{For} device $k\in S_t$ \textbf{in parallel do:}\\
\qquad\qquad \textbf{For} e = 1 to E \textbf{do:}\\
\qquad\qquad\qquad \textbf{For all} example $A_{t}^{(i_k)}$, $i_k\in\{1,2,\cdots,m_k\}$ \textbf{do:}\\
\qquad\qquad\qquad\qquad $\mathbff{z}{t}{k} =\mathbff{w}{t-1}{k} - \alpha\nabla F_k(\mathbff{w}{t-1}{k};A_{t}^{(i_k)})$ \\
\qquad\qquad\qquad\qquad \textbf{If} the device receives $\overline{\mathbf{z}}_t$ from the server \textbf{then do:} \\
\qquad\qquad\qquad\qquad\qquad Generate $\mathbff{u}{t}{k}$ using \eqref{eq:definition of u}\\
\qquad\qquad\qquad\qquad\qquad $\mathbff{w}{t}{k} =  \mathbff{u}{t}{k}\odot\overline{\mathbf{z}}_{t} + (\mathbf{1}- \mathbff{u}{t}{k})\odot\mathbff{z}{t}{k}$ \\
\qquad\qquad\qquad\qquad\qquad Compute $q_t^{(k)}$ using \eqref{fl:eq100}\\
\qquad\qquad\qquad\qquad \textbf{Else do:}  \\
\qquad\qquad\qquad\qquad\qquad $\mathbff{w}{t}{k} = \mathbff{z}{t}{k}$ \\
\qquad\qquad\qquad\qquad\qquad $q_t^{(k)} = q_{t-1}^{(k)}$ \\
\qquad\qquad\qquad\qquad \textbf{End If}  \\
\qquad\qquad\qquad\qquad $t = t + 1$\\
\qquad\qquad\qquad \textbf{End For}\\
\qquad\qquad \textbf{End For}\\
\qquad\qquad Send $\mathbf{z}_t^{(k)}$ to the server with probability $q_t^{(k)}$\\
\qquad \textbf{End For}\\
\textbf{End While}\\
\hline
\textbf{Output:} $\widehat{\mathbf{w}}_t = \sum_{k}\mathbf{w}^{(k)}_t$\\
\hline
\end{tabular}
}
\end{table}

In the above example, since the data distributions of devices are known as a priori, we can prevent the performance degradation of the global model by excluding local updates of the second group in the update of the global model. In general, however, it is very difficult for the server to exclude those biased local updates since the local datasets are not revealed to the server due to the privacy of the user data. Instead of making the server to exclude biased local updates, we modify SAFL such that each device can decide whether to upload its local update to the server or not. This decision is done by measuring the performance gap between global and local models. 
Let $h(\overline{\mathbf{z}}_{t-1})$ and $h(\mathbf{z}^{(k)}_t)$ be the accuracies of the global evaluation model and the local model in the $k$-th device, respectively, then the performance gap $\Delta_{t}^{(k)}$ between the global and local models is defined as
\begin{equation}
\Delta_t^{(k)} = \frac{|h(\overline{\mathbf{z}}_{t-1}) - h(\mathbf{z}^{(k)}_t)|}{h(\overline{\mathbf{z}}_{t-1}) + h(\mathbf{z}^{(k)}_t) +\epsilon}, 
\label{fl:eq101}
\end{equation}
where $\epsilon$ is a small constant to avoid division by zero (e.g., $\epsilon = 10^{-6}$).
If $\Delta_t^{(k)}$ is large, then we consider the local update as a biased update and do not upload the local update to the server. To do so, we set the probability $q_{t}^{(k)}$ that the local update $\mathbf{z}_{t}^{(k)}$ is uploaded to the server as\footnote{The choice of the exponential decay is based on our empirical experiences.}
\begin{equation}
q_t^{(k)} = \exp\left(-\frac{\Delta_t^{(k)}}{\nu}\right),
\label{fl:eq100}
\end{equation}
where $\nu$ is a regularization parameter.
Since the probability $q_t^{(k)}$ decays exponentially with the performance gap $\Delta_t^{(k)}$, if $\Delta_t^{(k)}$ is large, then it is highly likely that the device does not send its local update $\mathbf{z}_{t}^{(k)}$ to the server. By excluding the biased local update $\mathbf{z}_{t}^{(k)}$ in the update of the global model, we can prevent the performance degradation of the global model.

While the communication cost of SAFL is the same as that of FedAvg, the extended SAFL can reduce the number of local updates uploaded to the server. Let $X$ be the total local updates of $n$ devices after $T$ communication rounds and let $X_t^{(k)}$ be the random variable indicating whether the $k$-th device sends the local update to the server, i.e., $P(X_t^{(k)} = 1) = q_t^{(k)}$ and $P(X_t^{(k)} = 0) = 1-q_t^{(k)}$.
Then, we have $X  = \sum\limits_{t=1}^T \sum\limits_{k=1}^n X_t^{(k)}$ and thus
\begin{align}
E[X] & = \sum\limits_{t=1}^T \sum\limits_{k=1}^n \exp\left(-\frac{\Delta_t^{(k)}}{\nu}\right) \leq \sum\limits_{t=1}^T \sum\limits_{k=1}^n 1 = nT,
\end{align}
where $nT$ is the total local updates of FedAvg.

In Algorithm II, we summarize the extended SAFL algorithm.


\begin{table}[t!]
\centering
\caption{Deep Neural Networks}
\label{fl:tab002}
\begin{tabular}{|c|c|c|c|}
\hhline{|-|-|-|-|}
 \multicolumn{2}{|c|}{\cellcolor{gray!30} LeNet-5} & \multicolumn{2}{|c|}{ \cellcolor{gray!30} Light VGGNet}    \\
 \hhline{|-|-|-|-|}
  Layer & Filter Stride & Layer  &  Filter Stride\\
\hhline{|-|-|-|-|}
 conv5-6 &  1 		& conv3-64 			& 1 \\
 avg-pool-2 & 2 	& conv3-128  		& 1 \\
 conv5-16 & 1 		& max-pool-2  		& 2 \\
 avg-pool-2 & - 	& conv3-128  		& 1 \\
 FC-120 	& -		& max-pool-2  		& 2 \\
 FC-84 		& - 	& conv3-128  		& 1 \\
 FC-10  	& -		& max-pool-2  		& 2 \\
 softmax 	& -		& conv3-128  		& 1 \\
 - 			& -		& max-pool-2  		& 2 \\
 - 			& -		& global-avg-pool 	& - \\
 - 			& -		& conv1-10 			& 1 \\
 - 			& -		& softmax 			& - \\
\hhline{|-|-|-|-|}
\multicolumn{2}{|c|}{0.04M params}    & \multicolumn{2}{|c|}{1.76M params }	\\
\hhline{|-|-|-|-|}
\end{tabular}
\end{table}

\section{Simulation}
\label{fl:sec002}

\begin{figure*}[t]
\centering
\begin{adjustbox}{max size={!}{!}}
\begin{tikzpicture}
\node (c1) {\includegraphics[scale=0.277]{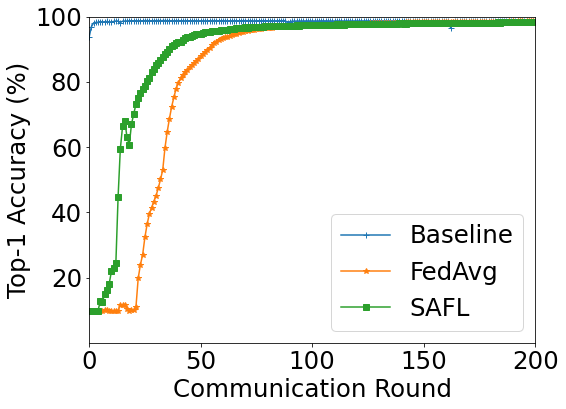}};

\node[right of=c1,xshift=4.5cm] (c2) {\includegraphics[scale=0.277]{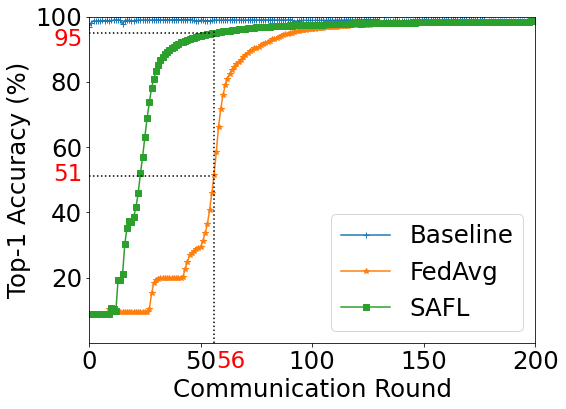}};

\node[right of=c2,xshift=4.5cm] (c3) {\includegraphics[scale=0.277]{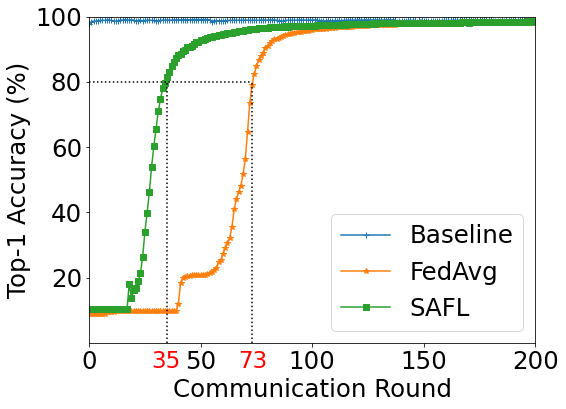}};

\node[below of=c1,yshift=-4cm] (c4) {\includegraphics[scale=0.277]{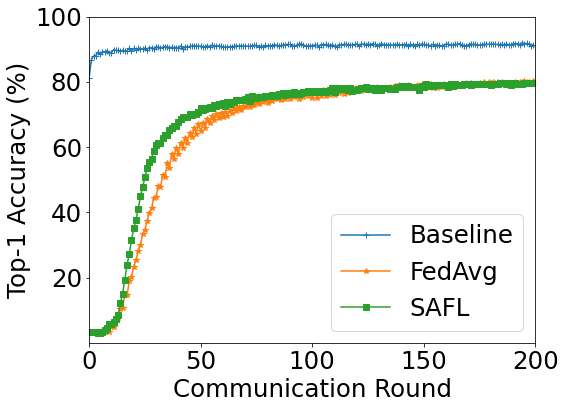}};

\node[right of=c4,xshift=4.5cm] (c5) {\includegraphics[scale=0.277]{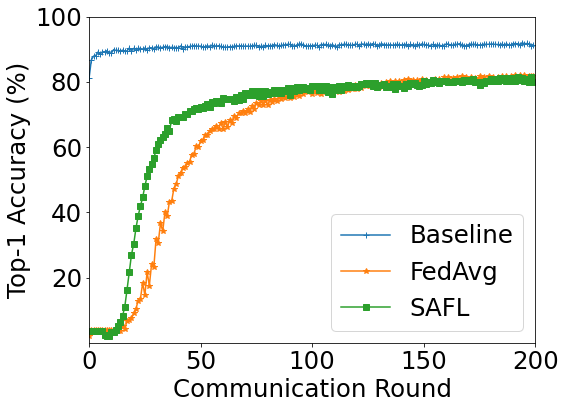}};

\node[right of=c5,xshift=4.5cm] (c6) {\includegraphics[scale=0.277]{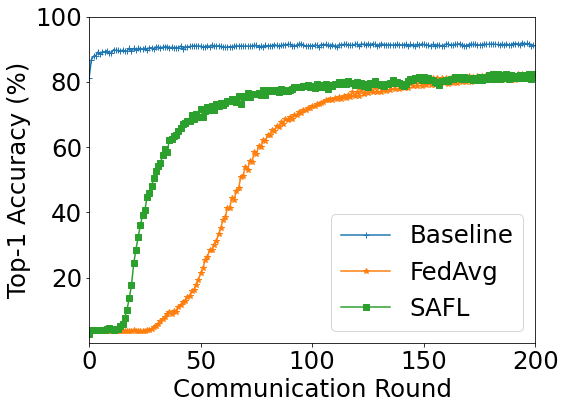}};

\node[below of=c1,yshift=3.1cm,xshift=0.5cm] (c11) {$n = 40$};
\node[below of=c2,yshift=3.1cm,xshift=0.5cm] (c21) {$n = 80$};
\node[below of=c3,yshift=3.1cm,xshift=0.5cm] (c31) {$n = 120$};

\node[below of=c2,yshift=-1.4cm,xshift=0.5cm] (c21) {(a)};
\node[below of=c5,yshift=-1.4cm,xshift=0.5cm] (c51) {(b)};

\end{tikzpicture}
\end{adjustbox}
\caption{Test accuracy of SAFL on different datasets: (a) MNIST and (b) GSC.}
\label{fl:fig003}
\end{figure*}

In this section, we investigate the empirical performance of the proposed SAFL on various benchmark datasets, which has been popularly used in the FL evaluation. 
We first summarize the datasets used in our experiments as follows:
\begin{itemize}
\item \textit{MNIST~\cite{lecunMNIST}:} a dataset consisting of $70,000$ images of handwritten digits between $0$ and $9$. All the images are divided into two groups: $60,000$ images for the training set and $10,000$ images for the test set\footnote{The training and the test sets are split by the command \textit{tf.keras.datasets.mnist.load\_data()} in Tensorflow.}.

\item \textit{Fashion-MNIST~\cite{xiao2017}:} a dataset containing 70,000 grayscale images of clothing (e.g., sneakers, shirts, shoes, and bags). These images are classified into 10 categories. The dataset is divided into two sets: the training set of 60,000 images and the test set of 10,000 images.  

\item \textit{CIFAR10~\cite{krizhevsky2009}:} a dataset of color images popularly used in image classification. It consists of 60,000 images of $32\times 32$ pixels from 10 categories: airplane, automobile, bird, cat, deer, dog, frog, horse, ship, and truck. The dataset is divided into two: training set of 50,000 images and test set of 10,000 images. 

\item \textit{Google speech commands dataset~\cite{warden2017}:} a dataset popularly used in speech recognition tasks. It consists of 65,000 utterances of 30 short words.
To pre-process the GSC dataset, we compute the first 13 mel-frequency cepstral coefficients (MFCC) of a speech signal using 80 filterbanks. To be specific, we first perform a 1024-point short-time Fourier transform (STFT) with frames of 64ms and 75\% overlap (at 16kHz sampling frequency) and then compute the power spectrum and MFCC.
  
\end{itemize}

\begin{table*}[t!]

\centering
\caption{Test accuracy in the 50-th communication round.}
{\normalsize
\begin{tabular}{c|c|c|c|c|c|c}
\hline
\multirow{2}{*}{Dataset} & \multicolumn{2}{c|}{Baseline} & \multicolumn{2}{c}{AvgFed} & \multicolumn{2}{|c}{SAFL}  \TBstrut \\
\cline{2-7}
& Top-1 Acc. & Top-5 Acc. & Top-1 Acc.  & Top-5 Acc. & Top-1 Acc.  & Top-5 Acc. \\
\hline
MNIST 		& 99\% & 99\% & 30\% & 90\% & 94\% & 99\% \\
FMNIST 		& 92\% & 99\% & 58\% & 98\% & 75\% & 99\% \\
CIFAR10 	& 77\% & 98\% & 21\% & 75\% & 35\% & 88\% \\
GSC 		& 91\% & 98\% & 62\% & 91\% & 72\% & 92\% \\
\hline
\end{tabular}
}
\label{fl:tab001}
\end{table*}

\begin{figure*}[t]
\centering
\begin{adjustbox}{max size={!}{!}}
\begin{tikzpicture}
\node (c1) {\includegraphics[scale=0.275]{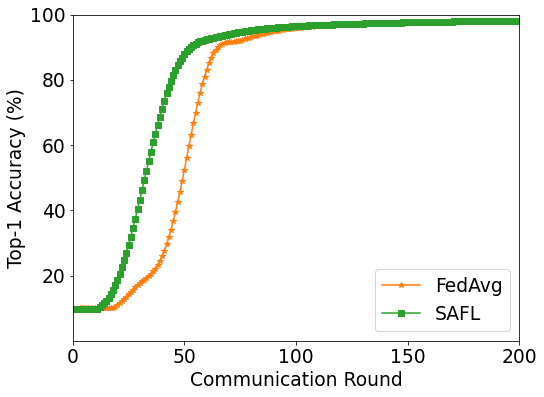}};
\node[below of=c1,yshift=-1.4cm,xshift=0cm] {(a)};

\node[right of=c1,xshift=4.5cm] (c2) {\includegraphics[scale=0.275]{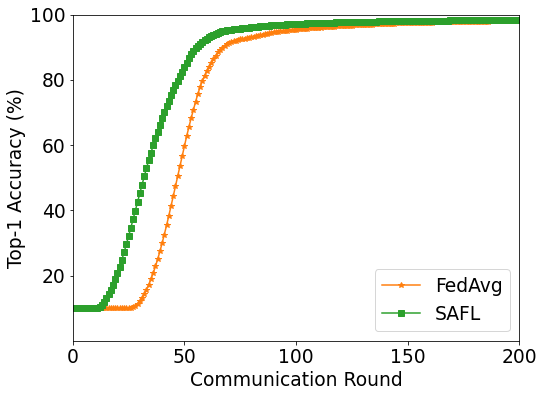}};
\node[below of=c2,yshift=-1.4cm,xshift=0cm] {(b)};

\node[right of=c2,xshift=4.5cm] (c3) {\includegraphics[scale=0.275]{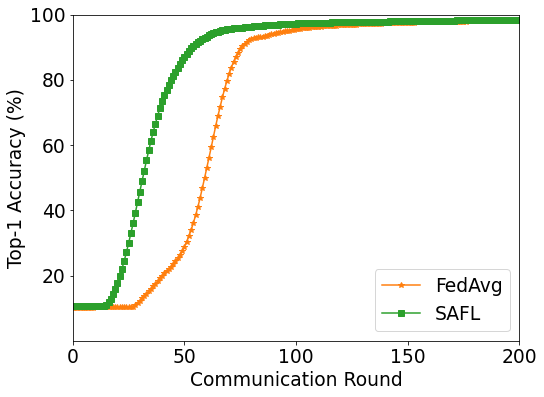}};

\node[below of=c3,yshift=-1.4cm,xshift=0cm] {(c)};

\end{tikzpicture}
\end{adjustbox}
\caption{Performance of SAFL for different fraction $R$ of selected devices: (a) $R = 0.3$, (b) $R = 0.5$, and (c) $R = 0.7$.}
\label{fl:fig103}
\end{figure*}
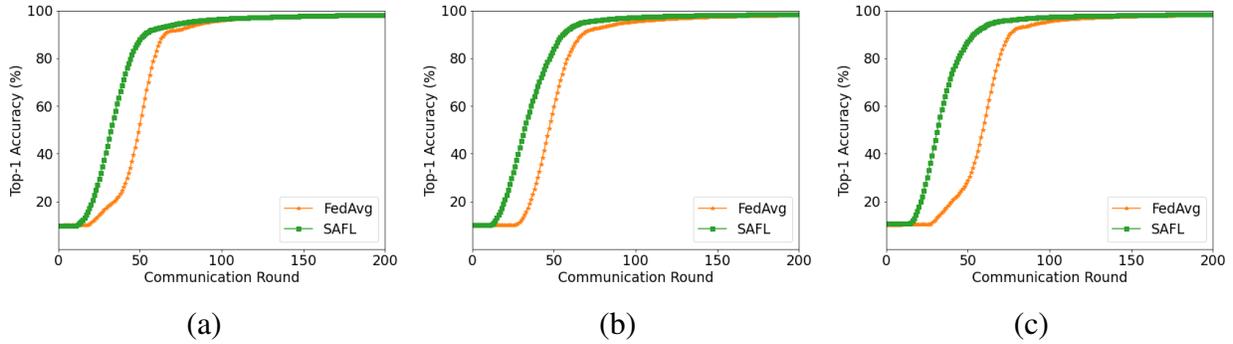 

\begin{table*}[t]

\centering
\caption{Effect of the parameters $\epsilon$ and $L$ in the proposed SAFL.}
{\normalsize
\begin{tabular}{c|c|c|c|c|c|c|c|c}
\hline
\multirow{2}{*}{$\epsilon$} & \multicolumn{2}{c|}{\shortstack{Baseline}} & \multicolumn{2}{c}{\shortstack{FedAvg}} & \multicolumn{2}{|c|}{\shortstack{SAFL \\ $L = 10$ }} & \multicolumn{2}{c}{\shortstack{SAFL \\ $L = 80$}} \TBstrut \\
\cline{2-9}
& \shortstack{Test \\ Cost} & \shortstack{Test Acc. } & \shortstack{Test \\ Cost}  & \shortstack{Test Acc. } & \shortstack{Test \\ Cost}  & \shortstack{Test Acc. } & \shortstack{Test \\ Cost}  & \shortstack{Test Acc. } \TBstrut \\
\hline
0.3   &  0.075   &  98.8\%  &   0.476  &   88.2\%  &   0.322  &   91.2\%  &    $\mathbf{0.186}$  &   $\mathbf{95.3}$\%\\
0.5   &  0.075   &  98.8\%   &  0.476   &  88.2\%  &   $\mathbf{0.206}$   &  $\mathbf{94.7}$\%  &   0.204  &   95.1\%\\
0.7    & 0.075    & 98.8\%   &  0.476   &  88.2\%  &   0.335   &  91.5\%  &   0.217   &  94.6\%\\
\hline
\end{tabular}
}
\label{fl:tab005}
\end{table*}

\begin{table*}[t!]
\caption{SAFL accuracy for different batch sizes.}
\centering
{\normalsize
\begin{tabular}{ c|c|c|c|c|c }
\hline
\multirow{2}{*}{Dataset size} & \multirow{2}{*}{FL Technique} & \multicolumn{4}{ c  }{Batch size} \\
\cline{3-6}
& & 50 & 100 & 150 & 200 \\
\hline
\multirow{2}{*}{Half of dataset} & FedAvg & 97.59\% & 97.15\% & 91.53\% & 93.03\%\\
& SAFL & \textbf{97.99\%} & \textbf{97.55\%} & \textbf{97.25\%} & \textbf{94.01\%}\\
\hline
\multirow{2}{*}{All of dataset} & FedAvg & 98.44\% & 97.42\% & 97.01\% & 95.74\%\\
& SAFL & \textbf{98.47\%} & \textbf{98.12\%}  & \textbf{97.67\%} & \textbf{97.08\%}\\
\hline
\end{tabular}
}
\label{fl:tab100}
\end{table*}

In our experiments, each training set is partitioned into $n$ subsets $(\mathcal{D}_1, \mathcal{D}_2,\cdots,\mathcal{D}_n)$ of user devices, and each of which is the local dataset in each device.
We consider the heterogeneous scenarios where $\mathcal{D}_k$ consists of just a few number of class labels (not all the labels) whose sizes are all different ($m_k = |\mathcal{D}_k|$).
To be specific, we first set $m_k = \max(\lfloor x_k \rfloor, 1)$ where $x_k\sim N(\overline{m}, \sigma^2)$ is a normal random variable with mean $\overline{m}$ and variance $\sigma^2$ and $\lfloor x_k \rfloor$ is the integer satisfying $\lfloor x_k \rfloor\leq x_k < \lfloor x_k \rfloor + 1$. 
Then we select a number of digits (e.g., at most 7 digits) for each device at random and choose $m_k$ samples randomly from the training subset containing only the selected digit labels.

For the MNIST classification, we use LeNet-5, a CNN model consisting of two sets of convolutional and pooling layers, followed by two fully-connected layers and the softmax classifier~\cite{lecun1998}. 
For the fashion-MNIST, CIFAR10, and GSC classifications, we use the VGGNet, a CNN model using only $3\times 3$ convolutional kernels~\cite{simonyan2015}.
The parameter settings of the CNN architectures are shown in Table \ref{fl:tab002}.
As a loss function in the training process, we use the cross-entropy:
\begin{equation}
H(\mathbf{y},\widehat{\mathbf{y}}) = -\sum\limits_{i=1}^{10} (y_i\ln \widehat{y}_i + (1-y_i)\ln (1-\widehat{y}_i)),
\end{equation}
where $\widehat{\mathbf{y}} = \matc{\widehat{y}_1 & \cdots & \widehat{y}_{10}}^T$ is the predicted softmax output and $\mathbf{y} = \matc{y_1 & \cdots & y_{10}}^T$ is the one-hot vector of the true label.
For all experiments, we set the learning rate $\alpha$ to a fixed constant ($\alpha = 0.002$) and set the number of local epochs to $E = 3$. 
We initialize the local model of each device with a different random seed.

\begin{table*}[t!]
\centering
\caption{Total local updates uploaded to the server.}
{\normalsize
\begin{tabular}{c|c|c|c|c|c|c|c}
\hline
\multirow{2}{*}{Iter} & \multicolumn{2}{c|}{\shortstack{FedAvg}} & \multicolumn{2}{c}{\shortstack{SAFL}} & \multicolumn{2}{|c}{\shortstack{Extended SAFL}}  \TBstrut \\
\cline{2-7}
& \shortstack{Total updates} & \shortstack{Test Acc. } & \shortstack{Total updates}  & \shortstack{Test Acc. } & \shortstack{Total updates}  & \multicolumn{1}{|c}{\shortstack{Test Acc. }}  \TBstrut \\
\hline
60   &  3050   &  52.22\%  &   3050   &  93.72\%  &   2023  &   \multicolumn{1}{|c}{94.83\%}   \\
70   &  3550   &  78.91\%   &  3550   &  95.35\%  &   2448   &  \multicolumn{1}{|c}{95.87\%}   \\
80    & 4050    & 92.32\%   &  4050   &  96.22\%  &   2883   &  \multicolumn{1}{|c}{96.51\%}   \\
\hline
\end{tabular}
}
\label{fl:tab101}
\end{table*}

We first evaluate the test accuracy of SAFL for different network size ($n = 40, 80, \text{ and } 120$).  
In this experiment, we set the parameters $\epsilon = 0.3$, $L = 80$, $\overline{m} = 600$, and $\sigma^2 = 100$. 
In Fig. \ref{fl:fig003}, we plot the test accuracy of SAFL and the conventional FedAvg as a function of the communication round.
The baseline is the centralized machine learning using the whole dataset.
From the results, we observe that the accuracy of all the FL algorithms improves  after a sufficient communication rounds (e.g., 100 rounds) and the performance of all FL algorithms eventually converges to the accuracy of the centralized learning technique. 
In particular, the proposed SAFL outperforms the standard FL technique by a large margin. 
For example, for $n = 80$, SAFL achieves the test accuracy of 95\% at the 56-th communication round, resulting in an accuracy improvement of more than 50\% (see Fig. \ref{fl:fig003}a).
We also observe that the proposed SAFL converges faster than the standard FL technique.
For example, when $n = 120$, SAFL achieves the accuracy of 80\% in 35 communication rounds, while the standard FL technique requires more than 70 rounds to achieve the same level of accuracy (see Fig. \ref{fl:fig003}a). Similar results can be observed from the GSC dataset (see Fig. \ref{fl:fig003}b). 
In Table. \ref{fl:tab001}, we show the top-1 and top-5 accuracy of SAFL in the early stage of SA for $n = 80$. From these experiments, we observe that SAFL outperforms the conventional approaches, resulting in an 17\% improvement of the top-1 accuracy on the FMNIST dataset.



We next examine the impact of the hyperparameters $\epsilon$ and $L$ on the performance of SAFL.
In this MNIST experiment, we set $n = 50$ and run simulations for different values $\epsilon = 0.3, 0.5,\text{ and } 0.7$.
In Table. \ref{fl:tab005}, we show the test cost and the test accuracy evaluated at the $50$-th communication round.
The best performance of SAFL is highlighted with bold digits.
For example, when $\epsilon = 0.3$ and $L = 80$, SAFL achieves the smallest MSE (i.e., MSE = $0.186$) and the best accuracy (i.e., 95.3\%).

We test the performance of SAFL for different training batch sizes ($B = 50$, 100, 150, and 200). For all the MNIST experiments, we set the parameters $n = 100$, $\epsilon = 0.3$, $L = 80$, $\overline{m} = 600$, $\sigma^2 = 100$. The MNIST accuracy are tested after $T = 100$ communication rounds. 
From the results, we observe that the small and moderate batch size can be used to enhance the accuracy of the FL networks, especially when the data size is reduced by half. For example, the batch size $B = 50$ gives more than  97.99\% SAFL accuracy while the batch size $B = 200$ results in less than 97.08\% accuracy (see Table. \ref{fl:tab100}).

\begin{figure}[t!]
\centering
\begin{adjustbox}{max width=1\textwidth, max totalheight=\textheight,keepaspectratio}
\begin{tikzpicture}
\node (c0) {\includegraphics[scale=0.42]{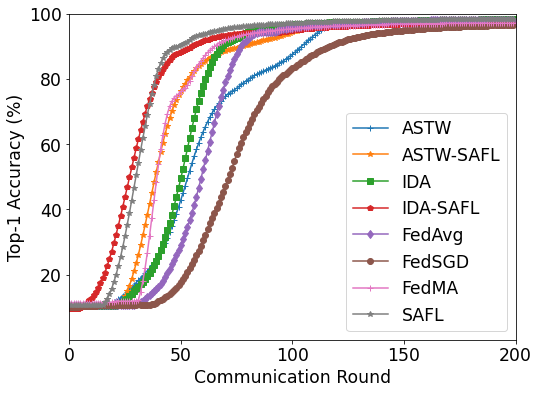}};


\end{tikzpicture}
\end{adjustbox}
\caption{The learning performance of the FL techniques.}
\label{fl:fig101}
\end{figure}

We also test the accuracy of SAFL for different fraction of selected devices ($R = 0.3$, 0.5, and 0.7). 
From the results, we observe that SAFL outperforms FedAvg, resulting in more than 50\% improvement of the test accuracy after 50 communication rounds when $R = 0.7$ (see Fig. \ref{fl:fig103}).

Next we evaluate the test accuracy of the extended SAFL as a function of the total local updates uploaded to the server. Here, we set $n = 100$ local devices and count the total local updates in different iterations ($T = 60$, 70, and 80). We run 100 trials and compute the mean values (see Table \ref{fl:tab101}). 
From the results, we observe that SAFL has the same communication cost as FedAvg.
While the accuracy of the extended SAFL is comparable to the SAFL accuracy, the
extended SAFL significantly reduces the number of the local updates uploaded by the devices, resulting in more than 30\% reduction of the local updates (see Table \ref{fl:tab101}).

Finally, we compare the performance of SAFL with the state-of-the-art FL techniques including the temporally weighted aggregation asynchronous (ASTW)~\cite{chensun2019}, IDA~\cite{yeganehfarshad2020}, FedAvg, FedSGD~\cite{chenmonga2016}, and FedMA~\cite{wangyurochkin2020}. 
We also test the combined algorithms: ASTW-SAFL and IDA-SAFL which are combined version of SAFL and ASTW/IDA fusion models~\cite{jisaravirta2021}.
From the results, we observe that SAFL outperforms FedSGD and FedAvg by a large margin, resulting in more than 50\% improvement of the test accuracy after 50 communication rounds. The performance of SAFL is comparable to that of FedMA.
We also observe that the combination of SAFL and state-of-the-art data fusion model can boost up the learning accuracy significantly. For example, IDA-SAFL can achieve more
than 80\% accuracy after 50 communication rounds, resulting in more than 30\% improvement of the test accuracy over the conventional IDA.

\section{Conclusion}
In this paper, we proposed a FL technique that greatly improves the accuracy and convergence speed of FL. Motivated by the observation that the average-based global model is not necessarily better than local models, the proposed SAFL technique allows each device to choose its own model instead of the global model in the early stage of FL. From the convergence analysis, we showed that SAFL sublinearly converges to the optimal solution under suitable conditions. Also, from the numerical experiments based on various benchmark datasets, we demonstrated that SAFL outperforms the conventional FL technique in terms of the convergence speed and the classification accuracy.   
In this work, we restricted our attention to the single-task learning scenario.
Our future work will be directed toward the extension to the multi-tasking scenario~\cite{smith2017}.

\begin{appendices}

\section{Proof of $\mathbf{w}^{(1)}$, $\mathbf{w}^{(2)}$, and $\mathbf{w}_\ast$}
\label{fl:apxF}
\begin{proof}
We first find the solution $\mathbf{w}^{(1)}$. 
Let $(\mathbf{x}_1, y_1) = (\matc{\frac{1}{4} & 0}^T, -1)$ and $\mathbf{w} = \matc{w_1 & w_2}^T$.
Then, we have
\begin{eqnarray}
\mathbf{w}^{(1)} & = & \text{arg}\min\limits_\mathbf{w} \:\: J(\mathbf{w},\mathcal{D}_1)\nonumber\\
& = & \text{arg}\min\limits_\mathbf{w} \:\: (y_1 - \mathbf{x}_1^T\mathbf{w})^2 + \|\mathbf{w}\|_1 \nonumber\\
 & = & \text{arg}\min\limits_{\mathbf{w}} \:\: (-1 - \frac{1}{4} w_1)^2 + |w_1| + |w_2|\nonumber\\
  & = & \matc{\text{arg}\min\limits_{w_1} \:\: (-1 - \frac{1}{4} w_1)^2 + |w_1| \\
  \text{arg}\min\limits_{w_2} \:\: |w_2|}\nonumber\\
  & \overset{(a)}{=} & \matc{0 \\ 0}\nonumber ,
\label{fl:eq102}
\end{eqnarray}
where (a) is because $(-1 - \frac{1}{4} w_1)^2 + |w_1| = 1 + \frac{1}{16}w_1^2 + \frac{1}{2}w_1 + |w_1|\geq 1 + \frac{1}{16}w_1^2 \geq 1$ and the equality holds if and only if $w_1=0$.  
Similarly, we can find out the solutions $\mathbf{w}^{(2)}=\matc{0 & \frac{4}{9}}^T$ and $\mathbf{w}_\ast=\matc{0 & \frac{4}{9}}^T$, which is the desired results. 

%
%

\end{proof}

\section{Proof of Lemma \ref{fl:lm002}}
\label{fl:apxA}
\begin{proof}
In this proof, we first show a recursive inequality of the MSE and then build the upper bound of the MSE.

Let $\mathbff{e}{t}{k} = \mathbff{w}{t}{k} - \mathbf{w}_\ast$ and $\Delta \mathbff{z}{t}{k} = \mathbff{z}{t}{k} - \overline{\mathbf{z}}_{t}$, then from \eqref{eq:local update formula}, and \eqref{eq:local training formula}, we have
\begin{align}
\mathbff{e}{t}{k} & =   \mathbff{w}{t}{k} - \mathbf{w}_\ast \nonumber\\
& =   \delta_t\mathbff{u}{t}{k}\odot\overline{\mathbf{z}}_{t}  + (\mathbf{1}-\delta_t\mathbff{u}{t}{k})\odot\mathbff{z}{t}{k} - \mathbf{w}_\ast  \nonumber\\
& =   \mathbff{z}{t}{k} - \mathbf{w}_\ast - \delta_t\mathbff{u}{t}{k}\odot\Delta \mathbff{z}{t}{k}    \nonumber\\
& =   (\mathbff{w}{t-1}{k} - \mathbf{w}_\ast) - \delta_t\mathbff{u}{t}{k}\odot\Delta \mathbff{z}{t}{k}   - \alpha  \nabla F_k(\mathbff{w}{t-1}{k};A_{t}^{(i_k)})\nonumber\\
 & =   \mathbff{e}{t-1}{k} - \delta_t\mathbff{u}{t}{k}\odot\Delta \mathbff{z}{t}{k} - \alpha(\nabla F_k(\mathbff{w}{t-1}{k}) - \nabla F_k(\mathbf{w}_\ast))  
 + \alpha(\nabla F_k(\mathbff{w}{t-1}{k}) - \nabla F_k(\mathbff{w}{t-1}{k};A_{t}^{(i_k)})) \nonumber.
\end{align}
Applying Taylor's expansion yields
\begin{align}
\mathbff{e}{t}{k} & =   \mathbff{e}{t-1}{k} - \delta_t\mathbff{u}{t}{k}\odot\Delta \mathbff{z}{t}{k} - \alpha \nabla^2 F(\mathbff{\boldsymbol\xi}{t-1}{k})\mathbff{e}{t-1}{k}    
 + \alpha(\nabla F_k(\mathbff{w}{t-1}{k}) - \nabla F_k(\mathbff{w}{t-1}{k};A_{t}^{(i_k)})) \nonumber\\
& =   (\mathbf{I} - \alpha \nabla^2 F_k(\mathbff{\boldsymbol\xi}{t-1}{k}))\mathbff{e}{t-1}{k} - \delta_t\text{Diag}(\mathbff{u}{t}{k})\Delta \mathbff{z}{t}{k}   
 + \alpha(\nabla F_k(\mathbff{w}{t-1}{k}) - \nabla F_k(\mathbff{w}{t-1}{k};A_{t}^{(i_k)})) \nonumber , 
\end{align}
where $\mathbff{\boldsymbol\xi}{t-1}{k}$ is a point in the line segment of two endpoints $\mathbff{w}{t-1}{k}$ and $\mathbf{w}_\ast$.

Taking the conditional variance of $\mathbff{e}{t}{k}$, we have 
\begin{align}
 E[\|\mathbff{e}{t}{k}\|_2^2|\mathbff{w}{t-1}{k},\Delta \mathbff{z}{t}{k},\mathbff{u}{t}{k}]
 & =  \| E[\mathbff{e}{t}{k}|\mathbff{w}{t-1}{k},\Delta \mathbff{z}{t}{k},\mathbff{u}{t}{k}]\|_2^2 
 + tr(Var(\mathbff{e}{t}{k}|\mathbff{w}{t-1}{k},\Delta \mathbff{z}{t}{k},\mathbff{u}{t}{k})) \nonumber\\
& =  \| \mathbf{A}\mathbff{e}{t-1}{k} - \delta_t\text{Diag}(\mathbff{u}{t}{k})\Delta \mathbff{z}{t}{k} \|_2^2 
  + \alpha^2tr(Var(\nabla F_k(\mathbff{w}{t-1}{k}) \nonumber\\
& \quad - \nabla F_k(\mathbff{w}{t-1}{k};A_{t}^{(i_k)})|\mathbff{w}{t-1}{k},\Delta \mathbff{z}{t}{k},\mathbff{u}{t}{k})) \nonumber,
\end{align}
where $\mathbf{A} = \mathbf{I} - \alpha  \nabla^2 F_k(\mathbff{\boldsymbol\xi}{t-1}{k})$.

By Assumption \textbf{A3}, the stochastic gradient has a bounded variance:
\begin{align}
tr(Var(\nabla F_k(\mathbff{w}{t-1}{k}) - \nabla F_k(\mathbff{w}{t-1}{k};A_{t}^{(i_k)})|\mathbff{w}{t-1}{k})) \leq \sigma_k^2.
\end{align} 
Thus, we have
\begin{align}
 E[\|\mathbff{e}{t}{k}\|_2^2|\mathbff{w}{t-1}{k},\Delta \mathbff{z}{t}{k},\mathbff{u}{t}{k}]
& \leq  \| \mathbf{A}\mathbff{e}{t-1}{k} - \delta_t\text{Diag}(\mathbff{u}{t}{k})\Delta \mathbff{z}{t}{k} \|_2^2   + \alpha^2\sigma_k^2 \nonumber\\
& =  \| \mathbf{A}\mathbff{e}{t-1}{k}\|_2^2 + \delta_t\|\text{Diag}(\mathbff{u}{t}{k})\Delta \mathbff{z}{t}{k} \|_2^2   + \alpha^2\sigma_k^2 \nonumber\\
&\quad  - 2\delta_t(\Delta \mathbff{z}{t}{k})^T\text{Diag}(\mathbff{u}{t}{k})\mathbf{A} \mathbff{e}{t-1}{k}\nonumber .
\end{align} 

Taking expectations of the last inequality, and noting the law of total expectation ($E[X]=E[E[X|Y]]$), we have
\begin{align}
 E[\|\mathbff{e}{t}{k}\|_2^2 ] 
& =   E[E[\|\mathbff{e}{t}{k}\|_2^2|\mathbff{w}{t-1}{k},\Delta \mathbff{z}{t}{k},\mathbff{u}{t}{k}]]\notag\\
& \leq   E[\| \mathbf{A}\mathbff{e}{t-1}{k}\|_2^2] + \delta_tE[\|\text{Diag}(\mathbff{u}{t}{k})\Delta \mathbff{z}{t}{k} \|_2^2]  + \alpha^2\sigma_k^2 \notag\\
&\quad - 2\delta_tE[(\Delta \mathbff{z}{t}{k})^T\text{Diag}(\mathbff{u}{t}{k})\mathbf{A}\mathbff{e}{t-1}{k}]  \notag.
\end{align}
By Assumption \textbf{A2}, the positive definite matrix $\mathbf{A}$ satisfies $\| \mathbf{A}\mathbff{e}{t-1}{k}\|_2^2\leq (1 - \alpha\mu)^2 \|\mathbff{e}{t-1}{k}\|^2_2$. It follows that
\begin{align}
 E[\|\mathbff{e}{t}{k}\|_2^2 ] 
& \leq   (1 - \alpha\mu)^2 E[\|\mathbff{e}{t-1}{k}\|^2_2] + \alpha^2\sigma_k^2 \notag\\
&\quad + \delta_tE[\|\text{Diag}(\mathbff{u}{t}{k})\Delta \mathbff{z}{t}{k} \|_2^2] 
 - 2\delta_tE[(\Delta \mathbff{z}{t}{k})^T\text{Diag}(\mathbff{u}{t}{k})\mathbf{A}\mathbff{e}{t-1}{k}]  .
\label{fl:eq400}
\end{align}
Using the law of total expectation, one can easily check that
\begin{align}
 E[\|\text{Diag}(\mathbff{u}{t}{k})\Delta \mathbff{z}{t}{k} \|_2^2]
& = E[E[\|\text{Diag}(\mathbff{u}{t}{k})\Delta \mathbff{z}{t}{k} \|_2^2 |\mathbff{w}{t-1}{k},\Delta \mathbff{z}{t}{k}]]\notag\\
\label{fl:eq401}
& = (1 - p(1 - \epsilon^2)) E[\| \Delta \mathbff{z}{t}{k}\|_2^2 ]\\
 E[(\Delta \mathbff{z}{t}{k})^T\text{Diag}(\mathbff{u}{t}{k})\mathbf{A}\mathbff{e}{t-1}{k}] 
& = E[E[(\Delta \mathbff{z}{t}{k})^T\text{Diag}(\mathbff{u}{t}{k})\mathbf{A}\mathbff{e}{t-1}{k} | |\mathbff{w}{t-1}{k},\Delta \mathbff{z}{t}{k}]]\notag\\
\label{fl:eq402}
& = (1 - p(1 - \epsilon)) E[(\Delta \mathbff{z}{t}{k})^T\mathbf{A}\mathbff{e}{t-1}{k} ]
\end{align} 
From \eqref{fl:eq400}, \eqref{fl:eq401}, and \eqref{fl:eq402}, we have
\begin{align}
 E[\|\mathbff{e}{t}{k}\|_2^2 ] 
& \leq   (1 - \alpha\mu)^2 E[\|\mathbff{e}{t-1}{k}\|^2_2] + \delta_t (1 - p(1 - \epsilon^2)) E[\| \Delta \mathbff{z}{t}{k}\|_2^2 ]  \notag\\
& \quad  - 2\delta_t (1 - p(1 - \epsilon)) E[(\Delta \mathbff{z}{t}{k})^T\mathbf{A}\mathbff{e}{t-1}{k} ] + \alpha^2\sigma_k^2 \notag \\
& =   (1 - \alpha\mu)^2 E[\|\mathbff{e}{t-1}{k}\|^2_2]  -\delta_t E[\Omega] + \alpha^2\sigma_k^2\notag\\
& \overset{(a)}{\leq}  (1 - \alpha\mu)^2 E[\|\mathbff{e}{t-1}{k}\|^2_2] + \alpha^2\sigma_k^2  \notag\\
&\quad -\delta_t (1 - p(1 - \epsilon^2))(1   - \alpha  (2\lambda - \mu))^2E[\|\mathbff{e}{t-1}{k}\|^2_2]\notag\\
& \overset{(b)}{\leq}   \tau_t E[\|\mathbff{e}{t-1}{k}\|^2_2]  + \alpha^2\sigma_k^2, \label{fl:eq024}
\end{align}
where $\Omega = 2 (1 - p(1 - \epsilon)) (\Delta \mathbff{z}{t}{k})^T\mathbf{A}\mathbff{e}{t-1}{k}  - (1 - p(1 - \epsilon^2)) \| \Delta \mathbff{z}{t}{k}\|_2^2 $, $\tau_t = (1-\alpha\mu)^2 - \delta_t(1 - p(1 - \epsilon^2))(1   - \alpha  (2\lambda - \mu))^2$, and (a) is because $E[\Omega]\geq (1 - p(1 - \epsilon^2))(1   - \alpha  (2\lambda - \mu))^2E[\|\mathbff{e}{t-1}{k}\|^2_2]$ (see Appendix \ref{fl:apxB}).

Using the recursion relationship (between $E[\|\mathbff{e}{t}{k}\|^2_2]$ and $E[\|\mathbff{e}{t-1}{k}\|^2_2]$) in \eqref{fl:eq024}, and noting that $\tau_t > 0$ as long as $\alpha\leq \frac{1}{2\lambda-\mu}$, we have 
\begin{align}
 E[\|\mathbff{e}{t}{k}\|_2^2 ] 
& \leq   \prod\limits_{j=1}^t \tau_jE[\|\mathbff{e}{0}{k}\|^2_2] + \alpha^2\sigma_k^2 (1 + \sum\limits_{i=0}^{t-2}\prod\limits_{j=0}^{i}\tau_{t-j})\notag \\
& =   (1-\alpha\mu)^{2t - 2\lfloor\frac{t}{q} \rfloor}\tau^{\lfloor\frac{t}{q} \rfloor}  E[\|\mathbff{e}{0}{k}\|^2_2] \notag\\
& \quad + \alpha^2\sigma_k^2 \sum\limits_{i=0}^{q-1}(1-\alpha\mu)^{2i}\sum\limits_{j=0}^{\lceil\frac{t+1}{q} \rceil - 1}\tau^j(1-\alpha\mu)^{2j(q-1)}\notag \\
& =    (1-\alpha\mu)^{2t - 2\lfloor\frac{t}{q} \rfloor}\tau^{\lfloor\frac{t}{q} \rfloor}  E[\|\mathbff{e}{0}{k}\|^2_2] \notag\\
& \quad + \alpha^2\sigma_k^2 \left(\frac{1 - (1-\alpha\mu)^{2q}}{1 - (1-\alpha\mu)^{2}}\right) \frac{1 - (\tau(1-\alpha\mu)^{2q-2})^{\lceil\frac{t+1}{q} \rceil }}{1 -  \tau(1-\alpha\mu)^{2q-2}}\notag \\
& \overset{(a)}{=}   (1-\alpha\mu)^{2t - 2\lfloor\frac{t}{q} \rfloor}\tau^{\lfloor\frac{t}{q} \rfloor}  E[\|\mathbff{e}{0}{k}\|^2_2]  \notag\\
&\quad +  \frac{\alpha^2\sigma_k^2}{1 - (1-\alpha\mu)^{2}} \frac{1 - (1-\alpha\mu)^{2q}}{1 -  \tau(1-\alpha\mu)^{2q-2}}\notag \\
& \overset{(b)}{\leq}   (1-\alpha\mu)^{2t}e^{-c\lfloor\frac{t}{q} \rfloor}  E[\|\mathbff{e}{0}{k}\|^2_2]  \notag\\
&\quad +  \frac{\alpha^2\sigma_k^2}{1 - (1-\alpha\mu)^{2}} \frac{1 - (1-\alpha\mu)^{2q}}{1 -  e^{-c}(1-\alpha\mu)^{2q}}\notag ,
\end{align}
where $q$ is the number of local iterations, $\tau = (1-\alpha\mu)^2 -  (1 - p(1 - \epsilon^2))(1   - \alpha  (2\lambda - \mu))^2$, $c = (1 - p(1 - \epsilon^2))(\frac{1   - \alpha  (2\lambda - \mu)}{1   - \alpha  \mu})^2$, (a) is because $1 - (\tau(1-\alpha\mu)^{2q-2})^{\lceil\frac{t+1}{q} \rceil }\leq 1$, and (b) is because $\tau \leq (1-\alpha\mu)^2e^{-c}$, which is the desired result.

\end{proof}

\section{Proof of Lemma \ref{fl:lm003}}
\label{fl:apxD}
\begin{proof}
We first show a recursive inequality of $\|E[\mathbff{w}{t}{k} - \mathbf{w}_\ast]\|_2$ and then build an upper bound of this term.

Let $\mathbff{e}{t}{k} = \mathbff{w}{t}{k} - \mathbf{w}_\ast$ and $\Delta \mathbff{z}{t}{k} = \mathbff{z}{t}{k} - \overline{\mathbf{z}}_{t}$, then from \eqref{eq:local update formula}, and \eqref{eq:local training formula}, we have
\begin{align}
 \|E[\mathbf{A}\mathbff{e}{t}{k}]\|_2 
& =   \|E[\mathbf{A}(\mathbff{w}{t}{k} - \mathbf{w}_\ast)]\|_2 \notag\\
& =   \|E[\mathbf{A}(\delta_t\mathbff{u}{t}{k}\odot\overline{\mathbf{z}}_{t}  + (\mathbf{1}-\delta_t\mathbff{u}{t}{k})\odot\mathbff{z}{t}{k} - \mathbf{w}_\ast)] \|_2 \notag\\
& \overset{(a)}{=}   \|\delta_t(\epsilon p + 1 - p)E[\mathbf{A}(\overline{\mathbf{z}}_{t}-\mathbf{w}_\ast)] 
 + (1-\delta_t(\epsilon p + 1 - p))E[\mathbf{A}(\mathbff{z}{t}{k} - \mathbf{w}_\ast)]\|_2  \nonumber\\
& =   \|\delta_t(\epsilon p + 1 - p)E[\mathbf{A}(\sum\limits_i\eta_i \mathbff{z}{t}{i} - \mathbf{w}_\ast)] 
 + (1-\delta_t(\epsilon p + 1 - p))E[\mathbf{A}(\mathbff{z}{t}{k} - \mathbf{w}_\ast)]\|_2  ,
\end{align}
where $\mathbf{A}$ is a matrix independent of $\mathbf{u}_t^{(k)}$ and (a) is because $E[u_j] = \epsilon p + 1 - p$ for all element $u_j$ of $\mathbf{u}_t^{(k)}$.

Applying Jensen's inequality yields
\begin{align}
\|E[\mathbf{A}\mathbff{e}{t}{k}]\|_2 
& \overset{(b)}{\leq}   \delta_t(\epsilon p + 1 - p)\sum\limits_i\eta_i \|E[ \mathbf{A}(\mathbff{z}{t}{i} - \mathbf{w}_\ast)]\|_2 \notag\\
& \quad  + (1-\delta_t(\epsilon p + 1 - p))\|E[\mathbf{A}(\mathbff{z}{t}{k} - \mathbf{w}_\ast)] \|_2 \notag\\
& \overset{(c)}{\leq}   \max\limits_i \|E[\mathbf{A}(\mathbff{z}{t}{i} - \mathbf{w}_\ast)] \|_2 ,
\label{fl:eq201}
\end{align}
where the last inequality is because $\sum_i\eta_i = 1$. 

Also, from the update expression \eqref{eq:local training formula}, we have 
\begin{align}
 E[\mathbf{A}(\mathbff{z}{t}{k} - \mathbf{w}_\ast)] 
& =   E[\mathbf{A}(\mathbff{w}{t-1}{k} - \alpha  \nabla F_k(\mathbff{w}{t-1}{k};A_{t}^{(j_k)}) - \mathbf{w}_\ast)] \notag\\
 & =   E[\mathbf{A}(\mathbff{e}{t-1}{k} -  \alpha(\nabla F_k(\mathbff{w}{t-1}{k}) - \nabla F_k(\mathbf{w}_\ast)))]  \notag\\
 & \quad + \alpha E[\mathbf{A}(\nabla F_k(\mathbff{w}{t-1}{k}) - \nabla F_k(\mathbff{w}{t-1}{k};A_{t}^{(j_k)}))] .
\end{align}
Using Taylor's expansion, we have
\begin{align}
 E[\mathbf{A}(\mathbff{z}{t}{k} - \mathbf{w}_\ast)] 
 & =   E[\mathbf{A}(\mathbff{e}{t-1}{k} -  \alpha \nabla^2 F(\mathbff{\boldsymbol\xi}{t-1}{k})\mathbff{e}{t-1}{k} )]  \nonumber\\
 & \quad + \alpha E[\mathbf{A}(\nabla F_k(\mathbff{w}{t-1}{k}) - \nabla F_k(\mathbff{w}{t-1}{k};A_{t}^{(j_k)}))] \notag\\
& =   E[\mathbf{A}\mathbf{G}^{(k)}_{t-1}\mathbff{e}{t-1}{k}] \notag\\
 & \quad + \alpha E[\mathbf{A}(\nabla F_k(\mathbff{w}{t-1}{k}) - \nabla F_k(\mathbff{w}{t-1}{k};A_{t}^{(j_k)}))] \notag,
\end{align}
where $\mathbf{G}^{(k)}_{t-1} = \mathbf{I} - \alpha \nabla^2 F_k(\mathbff{\boldsymbol\xi}{t-1}{k})$ and $\mathbff{\boldsymbol\xi}{t-1}{k}$ is a point in the line segment of two endpoints $\mathbff{w}{t-1}{k}$ and $\mathbf{w}_\ast$.

Using the law of total expectation yields
\begin{align}
E[\mathbf{A}(\mathbff{z}{t}{k} - \mathbf{w}_\ast)] 
& =   E[\mathbf{A}\mathbf{G}^{(k)}_{t-1}\mathbff{e}{t-1}{k}] 
 + \alpha E[\mathbf{A}(\nabla F_k(\mathbff{w}{t-1}{k}) - \nabla F_k(\mathbff{w}{t-1}{k};A_{t}^{(j_k)}))] \notag\\
& = E[\mathbf{A}\mathbf{G}^{(k)}_{t-1}\mathbff{e}{t-1}{k}] 
  + \alpha E[\mathbf{A}E[\nabla F_k(\mathbff{w}{t-1}{k}) - \nabla F_k(\mathbff{w}{t-1}{k};A_{t}^{(j_k)})|\mathbf{A},\mathbff{w}{t-1}{k}]] \notag\\
& = E[\mathbf{A}\mathbf{G}^{(k)}_{t-1}\mathbff{e}{t-1}{k}] ,
\label{fl:eq202}
\end{align}
where the last equality is because the input data $A_{t}^{(i_k)}$ is sampled identically and independently in each iteration and $E[\nabla F_k(\mathbff{w}{t-1}{k}) - \nabla F_k(\mathbff{w}{t-1}{k};A_{t}^{(j_k)})|\mathbf{A},\mathbff{w}{t-1}{k}] = 0.$


From \eqref{fl:eq201} and \eqref{fl:eq202}, we have 
\begin{eqnarray}
\|E[\mathbf{A}\mathbff{e}{t}{k}]\|_2 & \leq & \max\limits_i \|E[\mathbf{A}(\mathbff{z}{t}{i} - \mathbf{w}_\ast)] \|_2 \nonumber\\
& = & \max\limits_i E[\mathbf{A}\mathbf{G}^{(i)}_{t-1}\mathbff{e}{t-1}{i}] \nonumber.
\end{eqnarray}
Letting $j = \text{arg}\max\limits_i E[\mathbf{A}\mathbf{G}^{(i)}_{t-1}\mathbff{e}{t-1}{i}]$, we have $$\|E[\mathbf{A}\mathbff{e}{t}{k}]\|_2 \leq E[\mathbf{A}\mathbf{G}^{(j)}_{t-1}\mathbff{e}{t-1}{j}].$$ 
Applying this inequality yields
\begin{eqnarray}
\|E[\mathbff{e}{t}{k}]\|_2 & \leq &  \|E[\mathbf{G}^{(i_{t-1})}_{t-1}\mathbff{e}{t-1}{i_{t-1}}]\|_2 \nonumber\\
& \leq &   \|E[\mathbf{G}^{(i_{t-1})}_{t-1}\mathbf{G}^{(i_{t-2})}_{t-2}\mathbff{e}{t-2}{i_{t-2}}] \|_2\nonumber\\
& \leq & \|E[\mathbf{G}^{(i_{t-1})}_{t-1}\mathbf{G}^{(i_{t-2})}_{t-2}\mathbf{G}^{(i_{t-3})}_{t-3}\cdots\mathbf{G}^{(i_0)}_{0}\mathbff{e}{0}{i_0}]\|_2\nonumber\\
& \overset{(a)}{\leq} & E[\|\mathbf{G}^{(i_{t-1})}_{t-1}\cdots\mathbf{G}^{(i_0)}_{0}\mathbff{e}{0}{i_0}\|_2] \nonumber\\
& \overset{(b)}{\leq} & (1-\alpha\mu)^{t} E[\|\mathbff{e}{0}{i_0}\|_2] \nonumber\\
& \overset{(c)}{\leq} & (1-\alpha\mu)^{t} \sqrt{\zeta} ,
\label{fl:eq203}
\end{eqnarray}
where $i_{t-1} = \text{arg}\max\limits_i \|E[\mathbf{G}^{(i)}_{t-1}\mathbff{e}{t-1}{k}]\|_2$ and $i_{t-j} =   \text{arg}\max\limits_i \|E[\mathbf{G}^{(i_{t-1})}_{t-1}\cdots\mathbf{G}^{(i)}_{t-j}\mathbff{e}{t-j}{i}] \|_2$ ($j=2,3,\cdots,t$),  $\zeta = \max_i E[\|\mathbff{e}{0}{i}]\|_2^2]$, (a) is because of Jensen's inequality ($\|\gamma\mathbf{x}+(1-\gamma)\mathbf{y}\|_2\leq \gamma\|\mathbf{x}\|_2+(1-\gamma)\|\mathbf{y}\|_2$), (b) is because $\|\mathbf{G}^{(i)}_{t-1}\| \leq  1-\alpha\mu $, and (c) is because $\zeta \geq E[\|\mathbff{e}{0}{i_0}\|_2^2]\geq (E[\|\mathbff{e}{0}{i_0}\|_2])^2$, which is the desired result.


\end{proof}

\section{Proof of $\Omega\geq (1 - p(1 - \epsilon^2))(1   - \alpha  (2\lambda - \mu))^2\|\mathbff{e}{t-1}{k}\|^2_2$}
\label{fl:apxB}
\begin{proof}
In this proof, we show that $\Omega\geq (1 - p(1 - \epsilon^2))\|\mathbff{e}{t-1}{k}\|^2_2$ for some values of $\epsilon$ and $p$ as long as $\alpha < \frac{1}{2\lambda - \mu}$. In fact, we have
\begin{align}
\Omega   & =   2 (1 - p(1 - \epsilon)) (\Delta \mathbff{z}{t}{k})^T\mathbf{A}\mathbff{e}{t-1}{k} \notag\\
&\quad - (1 - p(1 - \epsilon^2)) \| \Delta \mathbff{z}{t}{k}\|_2^2  \notag\\
& =   \|\mathbff{e}{t-1}{k}\|_2^2  \left(2 (1 - p(1 - \epsilon)) \rho s \right. \notag\\
&\quad \left. - (1 - p(1 - \epsilon^2))s^2\right),\label{fl:eq026}
\end{align}
where $\rho = \frac{(\Delta \mathbff{z}{t}{k})^T\mathbf{A}\mathbff{e}{t-1}{k}}{\|\Delta \mathbff{z}{t}{k}\|_2\|\mathbff{e}{t-1}{k}\|_2}$ and $s = \frac{\| \Delta \mathbff{z}{t}{k}\|_2}{\|\mathbff{e}{t-1}{k}\|_2}$. Note that when $\|\mathbff{e}{t-1}{k}\|_2 = 0$ (i.e., $s\rightarrow\infty$), the algorithm already converges to the optimum $\mathbf{w}_\ast$. When $\|\Delta \mathbff{z}{t}{k}\|_2 = 0$ (i.e., $s = 0$), it is clear that $\Omega = 0$. So, we only need to consider the case of $0 < s < \infty$.    

First, we recall that for $\mathbf{x}$ and $\mathbf{y}$ satisfying $\|\mathbf{x}\|_2 = \|\mathbf{y}\|_2 = 1$, it follows $2\mathbf{x}^T\mathbf{A}\mathbf{y} = \mathbf{x}^T\mathbf{A}\mathbf{x} + \mathbf{y}^T\mathbf{A}\mathbf{y} - (\mathbf{x} - \mathbf{y})^T\mathbf{A}(\mathbf{x} - \mathbf{y}) \geq 2\lambda_{\min}  - \lambda_{\max} $ where $\lambda_{\max}$ and $\lambda_{\min} $ are the largest and the smallest eigenvalues of $\mathbf{A}$. Since $\mathbf{A} =  \mathbf{I} - \alpha  \nabla^2 F_k(\mathbff{\boldsymbol\xi}{t-1}{k})$, we have $\lambda_{\max} \leq 1- \alpha\mu$ and $\lambda_{\min} \geq 1- \alpha\lambda$. Therefore, we have 
\begin{equation}
\rho \geq \frac{1}{2}(2\lambda_{\min}(\mathbf{A}) - \lambda_{\max}(\mathbf{A}) \geq \frac{1}{2}(1   - \alpha  (2\lambda - \mu)).
\label{fl:eq025}
\end{equation}
Form \eqref{fl:eq026} and \eqref{fl:eq025}, we have
\begin{align}
\Omega & \geq   \|\mathbff{e}{t-1}{k}\|_2^2  \left((1 - p(1 - \epsilon)) (1   - \alpha  (2\lambda - \mu))s \right. \notag\\
&\quad \left. - (1 - p(1 - \epsilon^2))s^2\right)\nonumber .
\end{align}

Next, we have
\begin{align}
& \quad \frac{\Omega - \|\mathbff{e}{t-1}{k}\|_2^2 (1 - p(1 - \epsilon^2))(1   - \alpha  (2\lambda - \mu))^2}{\|\mathbff{e}{t-1}{k}\|_2^2 (s^2 + (1   - \alpha  (2\lambda - \mu))^2)} \notag\\
& =    (1 - p(1 - \epsilon))  \frac{(1   - \alpha  (2\lambda - \mu))s}{s^2 + (1   - \alpha  (2\lambda - \mu))^2}  - (1 - p(1 - \epsilon^2)) \notag \\
& \geq    (1 - p(1 - \epsilon))  c_1  - (1 - p(1 - \epsilon^2)) \notag \\
& \geq   g(\epsilon,p) \notag ,
\end{align}
where $c_1 = \min\limits_s\frac{(1   - \alpha  (2\lambda - \mu))s}{s^2 + (1   - \alpha  (2\lambda - \mu))^2}(\leq \frac{1}{2})$ and $g(\epsilon,p) = -p\epsilon^2 + pc_1\epsilon - (1-c_1)(1-p)$. Noting that $\frac{1}{2}\lambda\|\mathbff{e}{t-1}{k}\|^2_2 \geq f(\mathbff{w}{t-1}{k}) - f(\mathbf{w}_\ast) \geq \frac{1}{2}\mu\|\mathbff{e}{t-1}{k}\|^2_2$ and $f(\mathbf{w}_\ast) = 0$, we have 
\begin{align}
c_1  & =   \frac{(1   - \alpha  (2\lambda - \mu))s}{s^2 + (1   - \alpha  (2\lambda - \mu))^2} \notag\\
& \geq   \min\left(\frac{c_2\sqrt{\mu }(1   - \alpha  (2\lambda - \mu))}{c_2^2(1   - \alpha  (2\lambda - \mu))^2 + \mu } ,\right.\notag\\
&\qquad\qquad \left. \frac{c_2\sqrt{\lambda }(1   - \alpha  (2\lambda - \mu))}{c_2^2(1   - \alpha  (2\lambda - \mu))^2 + \lambda }\right), 
\end{align}
where $c_2 = \frac{1}{\|\Delta \mathbff{z}{t}{k}\|_2}\sqrt{2f(\mathbff{w}{t-1}{k})}$.

Now, what remains is to show $g(\epsilon,p)\geq 0$. In fact, we have
\begin{eqnarray}
g(\epsilon, p) & = & \frac{c_1^2}{4} - \frac{(1-c_1)(1-p)}{p} - (\epsilon - \frac{c_1}{2})^2.
\end{eqnarray}
It is not difficult to check that $g(\epsilon,p) \geq 0$ if 
\begin{eqnarray}
\label{fl:eq029}
\epsilon & = &  \frac{2(1-c_1)(1-p)}{c_1p}, \\
\label{fl:eq030}
p & \geq & \frac{4 (1 - c_1)}{(2 - c_1)^2} .
\end{eqnarray} 
Since $g(\epsilon,p) \geq 0$, we have $\Omega - \|\mathbff{e}{t-1}{k}\|_2^2 (1 - p(1 - \epsilon^2))(1   - \alpha  (2\lambda - \mu))^2 \geq 0$ which is the desired result.

\end{proof}

\section{Proof of Corollary \ref{fl:col001}}
\label{fl:apxAA}
\begin{proof}
Using Jensen's inequality, we have
\begin{eqnarray}
E[\|\widehat{\mathbf{w}}_t - \mathbf{w}_\ast\|_2^2] & = & E[\|\sum\limits_k \eta_k \mathbf{w}_t^{(k)} - \mathbf{w}_\ast\|_2^2] \nonumber\\
& \leq & \sum\limits_k\eta_k E[\|\mathbf{w}_t^{(k)} - \mathbf{w}_\ast\|_2^2]\nonumber\\
& \overset{(a)}{\leq} & \max\limits_k  E[\|\mathbf{w}_t^{(k)} - \mathbf{w}_\ast\|_2^2] \nonumber ,
\end{eqnarray}
where (a) is because $\sum_k\eta_k = 1$.
What remains is to show that if $\alpha_t = \frac{\alpha_0}{t+1}$, then
\begin{eqnarray}
E[\|\mathbf{w}_t^{(k)} - \mathbf{w}_\ast\|_2^2] \leq \frac{c}{t+1} .
\label{fl:eq207}
\end{eqnarray}
We will prove \eqref{fl:eq207} using the mathematical induction on $t$. First, since $c \geq \max_k E[\|\mathbf{w}_0^{(k)} - \mathbf{w}_\ast\|_2^2]$, it is clear that \eqref{fl:eq207} holds true for $t = 0$. 

Now we assume the induction hypothesis that \eqref{fl:eq207} holds true for $t-1$ and check if it also holds true for the case of $t$. Letting $\mathbff{e}{t}{k} = \mathbf{w}_t^{(k)} - \mathbf{w}_\ast$, and substituting $\alpha_t$ instead of $\alpha$ in \eqref{fl:eq024}, we have 
\begin{align}
E[\|\mathbff{e}{t}{k}\|_2^2 ] & \leq   (1 - \alpha_t\mu )^2 E[\|\mathbff{e}{t-1}{k}\|^2_2] + \alpha_t^2\sigma_k^2\notag \\
& \leq   (1 - \frac{\alpha_0\mu}{t+1}  )^2 \frac{c}{t} + \frac{\alpha_0^2\sigma_k^2}{(t+1)^2}\notag \\
& =   \frac{c}{t+1} \left( (1 - \frac{\alpha_0\mu}{t+1}  )^2 \frac{t+1}{t} + \frac{\alpha_0^2\sigma_k^2}{c(t+1)}\right) \notag \\
& \overset{(a)}{\leq}   \frac{c}{t+1} \left( (1 - \frac{\alpha_0\mu}{t+1}  )^2 \frac{t+1}{t} + \frac{2-(2-\alpha_0\mu)^2}{2(t+1)}\right) \notag \\
& =   \frac{c}{t+1} \frac{2t^2+t(2-\alpha_0^2\mu^2)+  2(1-\alpha_0\mu)^2}{2t(t+1)} \notag \\
& \overset{(b)}{<}   \frac{c}{t+1} \frac{2t^2+t(2-\alpha_0^2\mu^2)+  \alpha_0^2\mu^2}{2t(t+1)} \notag \\
& \overset{(c)}{\leq}   \frac{c}{t+1}  \notag ,
\end{align}
where (a) is because $c\geq \frac{2\alpha_0^2\sigma_k^2}{2-(2-\alpha_0\mu)^2}$, (b) is because $\frac{2-\sqrt{2}}{\mu} < \alpha_0 < \frac{2+\sqrt{2}}{\mu}$, and (c) is because $t\geq 1$, which is the desired result.

\end{proof}

\section{Proof of Theorem \ref{fl:thm003}}
\label{fl:apxAB}
\begin{proof}
We first find a recursive expression of $E[\sum\limits_k \eta_k \|\mathbf{w}_t^{(k)}-\mathbf{w}_\ast\|_2^2]$ and then prove by induction that
\begin{equation}
E[\sum\limits_k \eta_k \|\mathbf{w}_t^{(k)}-\mathbf{w}_\ast\|_2^2] \leq \frac{c}{t+1}.
\label{fl:eq215}
\end{equation}
From \eqref{eq:local update formula} and \eqref{eq:local training formula}, we have
\begin{align}
 E_\mathcal{U}[\sum\limits_k \eta_k \|\mathbf{w}_t^{(k)}-\mathbf{w}_\ast\|_2^2]
 & =   E_\mathcal{U}[\sum\limits_k \eta_k \|\delta_t \mathbff{u}{t}{k}\odot\overline{\mathbf{z}}_{t} + (\mathbf{1}-\delta_t \mathbff{u}{t}{k})\odot\mathbff{z}{t}{k}-\mathbf{w}_\ast\|_2^2]\notag\\
& =   E_\mathcal{U}[\sum\limits_k \eta_k \|\mathbff{z}{t}{k}-\mathbf{w}_\ast - \delta_t \mathbff{u}{t}{k}\odot(\mathbff{z}{t}{k} -\overline{\mathbf{z}}_{t} )\|_2^2]\notag\\
& =   \sum\limits_k \eta_k \|\mathbff{z}{t}{k}-\mathbf{w}_\ast\|^2_2 + \delta_t E_\mathcal{U}[ \sum\limits_k \eta_k \|\mathbff{u}{t}{k}\odot(\mathbff{z}{t}{k} -\overline{\mathbf{z}}_{t} )\|_2^2]\notag\\
& \quad  -\delta_t E_\mathcal{U}[\sum\limits_k \eta_k <2(\mathbff{z}{t}{k}-\mathbf{w}_\ast), \mathbff{u}{t}{k}\odot(\mathbff{z}{t}{k} -\overline{\mathbf{z}}_{t} )>] \notag\\
& =   \sum\limits_k \eta_k \|\mathbff{z}{t}{k}-\mathbf{w}_\ast\|^2_2 \notag\\
& \quad  + \delta_t E_\mathcal{U}[ \sum\limits_k \eta_k <\mathbff{u}{t}{k}\odot(\mathbff{z}{t}{k} -\overline{\mathbf{z}}_{t} )\notag\\
&\quad -  2(\mathbff{z}{t}{k}-\mathbf{w}_\ast), \mathbff{u}{t}{k}\odot(\mathbff{z}{t}{k} -\overline{\mathbf{z}}_{t} )>] \notag\\
& =   \sum\limits_k \eta_k \|\mathbff{z}{t}{k}-\mathbf{w}_\ast\|^2_2 \notag\\
& \quad + \delta_t E_\mathcal{U}[ \sum\limits_k \eta_k <\mathbff{u}{t}{k}\odot\mathbff{u}{t}{k}\odot(\mathbff{z}{t}{k} -\overline{\mathbf{z}}_{t} )\notag\\
&\quad -  2\mathbff{u}{t}{k}\odot(\mathbff{z}{t}{k}-\mathbf{w}_\ast),  \mathbff{z}{t}{k} -\overline{\mathbf{z}}_{t} >] \notag\\
& =   \sum\limits_k \eta_k \|\mathbff{z}{t}{k}-\mathbf{w}_\ast\|^2_2 \nonumber\\
& \quad + \delta_t  \sum\limits_k \eta_k <(p\epsilon^2+1-p)(\mathbff{z}{t}{k} -\overline{\mathbf{z}}_{t} )\notag\\
&\quad -  2(p\epsilon+1-p)(\mathbff{z}{t}{k}-\mathbf{w}_\ast),  \mathbff{z}{t}{k} -\overline{\mathbf{z}}_{t} > \notag\\
& \overset{(a)}{\leq}   \sum\limits_k \eta_k \|\mathbff{z}{t}{k}-\mathbf{w}_\ast\|^2_2 \notag\\
& \quad + \delta_t (p\epsilon+1-p) \sum\limits_k \eta_k (\|\mathbff{z}{t}{k} -\overline{\mathbf{z}}_{t}  -  (\mathbff{z}{t}{k}-\mathbf{w}_\ast)\|_2^2 \notag\\
&\quad - \|\mathbff{z}{t}{k} -\mathbf{w}_\ast\|_2^2) \notag\\
& =   \sum\limits_k \eta_k \|\mathbff{z}{t}{k}-\mathbf{w}_\ast\|^2_2 \nonumber\\
& \quad + \delta_t (p\epsilon+1-p)  (\|\overline{\mathbf{z}}_{t}  -  \mathbf{w}_\ast\|_2^2 - \sum\limits_k \eta_k \|\mathbff{z}{t}{k} -\mathbf{w}_\ast\|_2^2) \notag\\
& \overset{(b)}{\leq}   \sum\limits_k \eta_k \|\mathbff{z}{t}{k}-\mathbf{w}_\ast\|^2_2 \notag ,
\end{align}
where $E_\mathcal{U}[\mathbf{x}]$ is the expected value of $\mathbf{x}$ with respect to $\mathcal{U} = \{\mathbf{u}_t^{(k)}\}_{k}$, (a) is because $\epsilon \leq 1$ and $<\mathbf{a}-2\mathbf{b},\mathbf{a}> = \|\mathbf{a}-\mathbf{b}\|_2^2-\|\mathbf{b}\|_2^2$, and (b) is because $\|\overline{\mathbf{z}}_{t}  -  \mathbf{w}_\ast\|_2^2 = \|\sum\limits_k \eta_k  \mathbff{z}{t}{k}  -  \mathbf{w}_\ast\|_2^2 \leq \sum\limits_k \eta_k \|\mathbff{z}{t}{k} -\mathbf{w}_\ast\|_2^2$.

Taking expectation again, we have
\begin{align}
 E[\sum\limits_k \eta_k \|\mathbf{w}_t^{(k)}-\mathbf{w}_\ast\|_2^2] 
& \leq   E[\sum\limits_k \eta_k \|\mathbff{z}{t}{k}-\mathbf{w}_\ast\|^2_2]  \notag\\
& =   \sum\limits_k \eta_k E[\|\mathbff{w}{t-1}{k}-\mathbf{w}_\ast - \alpha_t\nabla F_k(\mathbff{w}{t-1}{k},A^{(i_k)}_t) \|^2_2]\notag\\
& =    \sum\limits_k \eta_k (E[\|\mathbff{w}{t-1}{k}-\mathbf{w}_\ast\|^2_2 \notag\\
&\quad - 2\alpha_t<\nabla F_k(\mathbff{w}{t-1}{k},A^{(i_k)}_t),\mathbff{w}{t-1}{k}-\mathbf{w}_\ast>]  \notag\\
& \quad  + \alpha_t^2E[\| \nabla F_k(\mathbff{w}{t-1}{k},A^{(i_k)}_t) \|_2^2]) \notag\\
& =    \sum\limits_k \eta_k(E[\|\mathbff{w}{t-1}{k}-\mathbf{w}_\ast\|^2_2 \notag\\
&  \quad - 2\alpha_t<E[\nabla F_k(\mathbff{w}{t-1}{k},A^{(i_k)}_t)|\mathbff{w}{t-1}{k}],\mathbff{w}{t-1}{k}-\mathbf{w}_\ast>]  \notag\\
& \quad + \alpha_t^2E[\| \nabla F_k(\mathbff{w}{t-1}{k},A^{(i_k)}_t) \|_2^2]) \notag\\
& \overset{(a)}{=}    \sum\limits_k \eta_k (E[\|\mathbff{w}{t-1}{k}-\mathbf{w}_\ast\|^2_2 \notag\\
&\quad - 2\alpha_t<\nabla F_k(\mathbff{w}{t-1}{k}),\mathbff{w}{t-1}{k}-\mathbf{w}_\ast>]  \notag\\
& \quad + \alpha_t^2E[\| \nabla F_k(\mathbff{w}{t-1}{k},A^{(i_k)}_t) \|_2^2]) \notag\\
& \overset{(b)}{\leq}    \sum\limits_k \eta_k ((1-\alpha_t\mu)E[\|\mathbff{w}{t-1}{k}-\mathbf{w}_\ast \|_2^2]   \notag\\
&\quad + \alpha_t^2E[\| \nabla F_k(\mathbff{w}{t-1}{k},A^{(i_k)}_t) \|_2^2])\notag\\
& \overset{(c)}{\leq}    (1-\alpha_t\mu) E[\sum\limits_k \eta_k \|\mathbff{w}{t-1}{k}-\mathbf{w}_\ast \|_2^2]   + \alpha_t^2 \sum_k\eta_k\sigma_k^2  ,
\label{fl:eq211}
\end{align}
where (a) is because $E[\nabla F_k(\mathbff{w}{t-1}{k},A^{(i_k)}_t)|\mathbff{w}{t-1}{k}] = \nabla F_k(\mathbff{w}{t-1}{k})$, (b) is due to \textbf{A4}, and (c) is due to \textbf{A5}.

Letting $\mathbf{e}_{t} =  E[\sum\limits_k \eta_k \|\mathbf{w}_t^{(k)}-\mathbf{w}_\ast\|_2^2]$, we will prove \eqref{fl:eq215} using induction on $t$. First, since $c \geq  E[\sum\limits_k \eta_k \|\mathbf{w}_0^{(k)}-\mathbf{w}_\ast\|_2^2]$, it is clear that \eqref{fl:eq215} holds true for $t = 0$. 

Now we assume the induction hypothesis that \eqref{fl:eq215} holds true for $t-1$ and check if it also holds true for the case of $t$. We have
\begin{eqnarray}
\mathbf{e}_t & \leq &  (1-\alpha_t\mu )  \mathbf{e}_{t-1}   +  \alpha_t^2 \sigma^2 \nonumber\\
& \leq &  (1 - \frac{\alpha_0\mu}{t+1}  )  \frac{c}{t} + \frac{\alpha_0^2\sigma^2}{(t+1)^2}\nonumber \\
& = & \frac{c}{t+1} \left(  (1 - \frac{\alpha_0\mu}{t+1}  )  \frac{t+1}{t} + \frac{\alpha_0^2\sigma^2}{c(t+1)}\right) \nonumber \\
& \overset{(a)}{\leq} & \frac{c}{t+1} \left( (1 - \frac{\alpha_0\mu}{t+1}  ) \frac{t+1}{t} + \frac{\alpha_0\mu-1}{t+1}\right) \nonumber \\
& = & \frac{c}{t+1} \frac{t(t+1)-(\alpha_0\mu - 1)}{t(t+1)} \nonumber \\
& \overset{(b)}{<} & \frac{c}{t+1}  \nonumber  ,
\end{eqnarray}
where $\sigma^2 = \sum_k\eta_k\sigma_k^2$, (a) is because $c\geq \frac{\alpha_0^2\sigma^2}{\alpha_0\mu-1}$ and (b) is because $\alpha_0 > \frac{1}{\mu}$, which is the desired result.

\end{proof}

\end{appendices}

\end{document}